\documentclass{article}

  \PassOptionsToPackage{numbers, compress}{natbib}

\usepackage{amsmath,amsfonts,bm}

\def\eqref#1{equation~\ref{#1}}

\def\1{\bm{1}}

\DeclareMathAlphabet{\mathsfit}{\encodingdefault}{\sfdefault}{m}{sl}
\SetMathAlphabet{\mathsfit}{bold}{\encodingdefault}{\sfdefault}{bx}{n}

\newcommand{\E}{\mathbb{E}}

\newcommand{\R}{\mathbb{R}}

\DeclareMathOperator*{\argmax}{arg\,max}

\usepackage[preprint]{neurips_2024}

\usepackage{microtype}
\usepackage{graphicx}
\usepackage{subfigure}
\usepackage{booktabs} 
\usepackage{multirow}
\usepackage{makecell}
\usepackage{algorithm}
\usepackage{framed}
\usepackage{algorithmic}
\usepackage{wrapfig}

\usepackage{hyperref}

\usepackage{amsmath}
\usepackage{amssymb}
\usepackage{mathtools}
\usepackage{amsthm}
\usepackage{fancybox}

\usepackage[capitalize,noabbrev]{cleveref}

\usepackage{enumitem}

\newtheorem{thm}{Theorem}
\newtheorem{prop}{Proposition}
\newtheorem{lem}{Lemma}

\newtheorem{defi}[thm]{Definition}
\newtheorem{ass}{Assumption}

\usepackage[utf8]{inputenc} 
\usepackage[T1]{fontenc}    
\usepackage{hyperref}       
\usepackage{url}            
\usepackage{booktabs}       
\usepackage{amsfonts}       
\usepackage{nicefrac}       
\usepackage{microtype}      
\usepackage{xcolor}         

\def \D {\mathcal{D}}
\def \A {\mathcal{A}}
\def \P {\mathcal{P}}
\def \B {\mathcal{B}}

\def \I {\mathbb{I}}

\def \S {\mathcal{S}}
\def\x{{\bf{x}}}
\def\t{{\bf{t}}}

\def\m{{\bf{m}}}

\def\w{{\bf{w}}}
\def\v{{\bf{v}}}

\def \u {{\bf{u}}}

\def \E {\mathbb{E}}
\def \R {\mathbb{R}}

\def \v {\textbf{v}}

\def \fair {\text{fair}}

\def \GCL {\text{GCL}}
\def \tx {\tilde{\x}}
\def \O {\mathcal{O}}

\title{Provable Optimization for Adversarial Fair Self-supervised Contrastive Learning}

\author{%
  Qi Qi \\
Department of Computer Science\\
  The University of Iowa\\
  Iowa City, IA 52242, USA \\ 
  \texttt{qi-qi@uiowa.edu} \\  
   \And
   Quanqi Hu \\ 
  Department of Computer Science  \& Engineering\\
  Texas A\&M University\\
  College Station, TX 77843, USA \\
  \texttt{quanqi-hu@tamu.edu} \\
   \AND
  Qihang Lin \\
Department of Computer Science\\
  The University of Iowa\\
  Iowa City, IA 52242, USA \\ 
  \texttt{qihang-lin@uiowa.edu} \\  
   \And
   Tianbao Yang \\
 Department of Computer Science  \& Engineering\\
  Texas A\&M University\\
  College Station, TX 77843, USA \\
  \texttt{tianbao-yang@tamu.edu} \\
}

\begin{document}

\setlength{\abovedisplayskip}{1pt}
\setlength{\belowdisplayskip}{1pt}

\maketitle

\begin{abstract}
This paper studies learning fair encoders in a self-supervised learning (SSL) setting, in which all data are unlabeled and only a small portion of them are annotated with sensitive attribute. 
 Adversarial fair representation learning is well suited for this scenario by minimizing a contrastive loss over unlabeled data while maximizing an adversarial loss of predicting the sensitive attribute over the data with sensitive attribute.  Nevertheless, optimizing adversarial fair representation learning presents significant challenges due to solving a non-convex non-concave minimax game. The complexity deepens when incorporating a global contrastive loss that contrasts each anchor data point against all other examples. A central question is ``{\it can we design a provable yet efficient algorithm for solving adversarial fair self-supervised contrastive learning}?'' Building on advanced optimization techniques, we propose a stochastic  algorithm dubbed SoFCLR with a convergence analysis under reasonable conditions without requring a large batch size. 
We conduct extensive experiments to demonstrate the effectiveness of the proposed approach for downstream classification with eight fairness notions. 
\end{abstract}

\vspace*{-0.1in}
\section{Introduction}
\vspace*{-0.1in}

Self-supervised learning (SSL) has become a pivotal paradigm in deep learning (DL), offering a groundbreaking approach to addressing challenges related to labeled data scarcity. The significance of SSL lies in its ability to leverage vast amounts of unlabeled data to learn encoder networks for extracting meaningful representations from input data that are useful for various downstream tasks. One state-of-the-art method for SSL is contrastive learning by minimizing a contrastive loss that contrasts a positive data pair with a number of negative data pairs.  It has shown remarkable success in pretraining encoder networks, e.g.,  Google's SimCLR model~\cite{chen2020simple} pretrained on image data  and OpenAI's CLIP model~\cite{radford2021learning} pretrained on image-text data, leading to improved performance when fine-tuned on downstream tasks.

However, like traditional supervised learning, SSL is also not immune to the fairness concern and biases inherent in the data. The potential for self-supervised models to produce unfair outcomes stems from the biases embedded in the unlabeled data used for training. The contrastive loss could inadvertently reinforce certain biases present in the data. For example, a biased feature that is highly relevant to the gender (e.g., long-hair) can be easily learned from the data to contrast a female face image with a number of images that are dominated by male face images. As a result, the learned  feature representations will likely to induce unfair predictive models for downstream tasks. One approach to address this issue is to utilize traditional supervised fairness-aware approaches to learn a predictive model in the downstream task by removing the disparate impact of biased features highly relevant to sensitive attributes. However, this approach will suffer from several limitations: (i)  it requires repeated efforts for different downstream tasks; (ii) it requires labeled data to be annotated with sensitive attribute as well, which may cause privacy issues. 

Several studies have put forward techniques to enhance the fairness of contrastive learning of representations~\cite{Park2022FairCL,DBLP:journals/corr/abs-2106-02866,zhang2023fairnessaware,10.1145/3543507.3583480}. However, the existing approaches focused on modifying the contrastive loss by restricting the space of positive data and/or the space of negative data of an anchor data by using the sensitive attributes of all data or an external image generator~\cite{zhang2023fairnessaware}.  {\bf Different from} the prior studies on fair contrastive learning, we revisit a classical idea of advesarial fair representation learning (AFRL) by solving the following minimax problem:
\begin{equation}\label{eqn:fair_minmax}
\begin{aligned}
    \min_\w\max_{\w'}F(\w, \w'):=F_{\text{GCL}}(\w) + \alpha  F_{\fair}(\w,\w'),
\end{aligned}
\end{equation}
where  $F_{\text{GCL}}(\w)$ denotes a self-supervised global contrastive loss (GCL) of learning an encoder network parameterized by $\w$ and  $ F_{\fair}(\w,\w')$ denotes an adversarial loss of a discrimnator parameterized by $\w'$ for predicting the sensitive attribute given the encoded representation. There are several benefits of this approach compared with existing fairness-aware contrastive learning~\cite{park2022fair,10.1145/3543507.3583480,DBLP:journals/corr/abs-2106-02866,zhang2023fairnessaware}. First, the contrastive loss remains intact. Hence, it offers more flexibility of choosing different contrastive losses including bimodal contrastive losses~\cite{radford2021learning}. Second, only the fairness-promoting regularizer depends on the sensitive attribute. Therefore, it is not necessary to annotate all unlabeled data with sensitive attribute, which makes it more suitable for SSL with a large number of unlabeled data but only limited data with sensitive attribute. 

Despite the simplicity of the framework, a challenging question remains: {\it is it possible to design an efficient algorithm that can be proved to converge for solving~(\ref{eqn:fair_minmax})?} There are two hurdles to be overcome: (i) the problem could be non-convex and non-concave if the discriminator is not a linear model, which makes the convergence analysis formidable; (ii) the standard mini-batch gradient estimator yields a biased stochastic gradient of the primal variable due to presence of GCL, which does not ensure the convergence. Although existing studies in ML have shown that solving a general non-convex minimax game (e.g., generative adversarial networks [GAN]) is not stable~\cite{kodali2018on},   it is still possible to develop provable yet efficient algorithms for AFRL because: (i) AFRL can use a simple/shallow network for the discriminator that operates on an encoded representation; hence one step of the dual update of $\w'$ followed by a step of the primal update of $\w$ is sufficient; (ii) non-convex minimax game has been proved to converge under weaker structured conditions than concavity. Based these two observations, we design an efficient algorithm and and provide  its convergence guarante. Our contributions are outlined below: 
\vspace*{-0.1in}
\begin{itemize}
\item[1.]Theoretically, we propose a stochastic algorithm for optimizing the non-convex compositional minimax game and manage to establish a convergence result under provable conditions without requiring convexity/concavity on either side. 
\item[2.] Empirically, we conduct extensive experiments to demonstrate the effectiveness of the proposed approach for downstream classification with eight fairness notions. 
\end{itemize}
\vspace*{-0.1in}
This paper is different from tremendous studies on fair representation learning, which either rely on the labeled data with annotated sensitive attribute, or simply use mini-batch based contrastive loss for SSL that has no convergence guarantee, or only examine few standard fairness measures for downstream classification. 

\begin{table*}[t] 
	\caption{Comparison with existing studies on fair deep representation learning. The column ``Label'' and ``Sensitive Attribute" mean whether the label information or the senstive attribute information is used in the training process. $^*$ means that the method can be extended to a setting with partial data annotated with sensitive attribute. }\label{tab:app_survey} 
	\centering 
	\label{tab:0} 
	\scalebox{0.8}{\begin{tabular}{l|c|cc||cccc}  
			\toprule 
Category&Reference &Adversarial&Contrastive& Label & Sensitive Attribute& \makecell{Sample\\ Generator
}&\makecell{Theoretical\\ Analysis}\\
\hline 
\multirow{7}{*}{(Semi-)Supervised} &\cite{DBLP:journals/corr/EdwardsS15}&Yes&No&Yes&Yes$^*$&No&No\\
&\cite{DBLP:journals/corr/LouizosSLWZ15}&No&No&Partial&Yes&No&No\\
&\citep{DBLP:journals/corr/BeutelCZC17}&Yes&No&Yes&Partial&No&No\\
&\cite{10.5555/3294771.3294827}&Yes&No&Yes&Yes$^*$&No&Yes\\
& \citep{Elazar2018AdversarialRO}  & Yes&No&Yes&Yes$^*$&No&No\\ 
&\citep{DBLP:journals/corr/abs-1904-05514}&Yes&No&Yes&Yes$^*$&No&Yes\\
& \cite{10.5555/3327546.3327583} & No&No&Yes& Yes&No&No\\
 &  \citep{Park2022FairCL} & No &Yes& Yes & Yes & No&No\\  
\hline
 \multirow{2}{*}{Unsupervised}  &\cite{DBLP:journals/corr/LouizosSLWZ15}&No&No&No&Yes&No&No\\
&\cite{10.5555/3327546.3327583} & No&No&No& Yes&No&No\\ 
\hline
\multirow{4}{*}{Self-supervised}& \cite{DBLP:journals/corr/abs-2106-02866}&No&Yes&No&Yes&No&No\\
    &\cite{chai2022selfsupervised} &No&Yes& Partial&No& No&No\\ 
&\cite{10.1145/3543507.3583480}&No&Yes& No&Yes& Yes&No\\
&\cite{zhang2023fairnessaware}&No&Yes&No&Partial&Yes&No\\
 &Our work& Yes&Yes&No&Partial&No&Yes\\
\bottomrule
	\end{tabular}} 
	\vspace*{-0.15in} 
\end{table*} 
\vspace*{-0.05in}

\section{Related Work}
\vspace*{-0.05in}

While there are tremendous work on fairness-aware learning~\cite{10.1145/3494672,DBLP:journals/corr/abs-1908-09635,barocas-hardt-narayanan, qi2023improving}, we restrict our discussion below to most relevant studies about fair representation learning, which is able to learn a mapping function that induces representations of data for fairer predictions. 

{\bf (Semi-)Supervised Fair Representation Learning.} The seminal work~\cite{pmlr-v28-zemel13} initiated the study of fair representation learning. The goal is to learn a fair mapping that maps the original input feature representation into a space that not only preserves the information of data but also obfuscate any information about membership in the protected group. Nevertheless, their approach is deemed as a shallow approach, which is not suitable for deep learning. 

For DL, many prior studies have considerd to learn an encoder network that induces fair representation of data with respect to sensitive attributes. A classical idea  is adversarial learning by minimizing the loss  of predicting class labels and maximizing the loss of predicting the sensitive attribute given the encoded representations. This has been studied in~\cite{DBLP:journals/corr/EdwardsS15,Elazar2018AdversarialRO,10.5555/3294771.3294827,DBLP:journals/corr/abs-1807-00199,DBLP:journals/corr/BeutelCZC17,10.5555/2946645.2946704,10.1145/3278721.3278779,DBLP:journals/corr/abs-1904-05514} for different applications or different contexts. For example, \cite{10.5555/2946645.2946704,DBLP:journals/corr/ChenASWC16} tackle the domain adaptation setting, where the encoder network is learned such that it cannot discriminate between the training (source) and test (target) domains. Other approaches have  explicitly considered fair classification with respect to some sensitive attribute. Among these studies, \cite{DBLP:journals/corr/BeutelCZC17} raised the challenge of collecting sensitive attribute information for all data and considered a setting only part of the labeled data have a sensitive attribute.  In addition to adversarial learning, variational auto-encoder (VAE) based methods have been explored for fair representation learning~\cite{DBLP:journals/corr/LouizosSLWZ15,gong2023practical,10.5555/3327546.3327583,DBLP:journals/corr/abs-2101-04108}. These methods can work in unsupervised learning setting where no labels are given or semi-supervised learning setting where only part of the data are labeled. However, they do not consider how to leverage data that are not annotated with attribute information. Moreover, VAE-based representation learning methods lag significantly  behind self-supervised contrastive representation learning  on complicated tasks in terms of performance~\cite{bizeul2023simvae}. 


{\bf Fair Contrastive Learning.} 
Another category of research that is highly related to our study is fair contrastive learning~\cite{park2022fair,10.1145/3543507.3583480,DBLP:journals/corr/abs-2106-02866,zhang2023fairnessaware}. These methods usually use the sensitive attribute information to restrict the space of negative and positive data in the contrastive loss. \cite{park2022fair} utilized a supervised contrastive loss for representation learning, and modifies the contrastive loss  by incorporating sensitive attribute information to encourage samples with the same class but different attributes to be closer together while pushing samples with the same sensitive attribute but different classes further apart. Their approach requires all data to be labeled and annotated with sensitive attribute information. 

Several papers have modified the contrastive loss for self-supervised learning~\cite{DBLP:journals/corr/abs-2106-02866, 10.1145/3543507.3583480, zhang2023fairnessaware}.  For example, in~\cite{DBLP:journals/corr/abs-2106-02866} the authors define a contrastive loss for each group of senstive attribute separately. \cite{10.1145/3543507.3583480} proposed to incorporate counterfactual fairness by using a counterfactual version of an anchor data as positive in contrastive learning. It is generated by  flipping the sensitive attribute (e.g., female to male) using a sample generator (cyclic variational autoencoder), which is learned separately.  
\cite{zhang2023fairnessaware} used a similar idea by constructing a positive sample of an achor data with a different sensitive attribute generated by the image attribute editor, and constructing the negative samples as the views generated from the different images with the same sensitive attribute. To this end, they also need to train an image attribute editor that can genereate a sample with a different sensitive attribute. 

\cite{chai2022selfsupervised} proposed a bilevel learning framework in the setting that no sensitive attribute information is available. They used a weighted self-supervised contrastive loss as the lower-level objective for learnig a representation and an averaged top-$K$ classification loss on validation data as the upper objective to learn the weights and the classifier.  Table~\ref{tab:app_survey} summarizes prior studies and our work.

\vspace*{-0.1in}\section{Preliminaries}
\vspace*{-0.1in}
\textbf{Notations.} Let $\D$ represent an unlabeled set of $n$ images, and let $\D_a\subset\D$ denote a subset of $k\ll n$ training images with attribute information.  
Let $\x\sim\D$ denote a random data uniformly sampled from $\D$. For each $(\x, a)\in\D_a$,  $a\in\{ 1, \ldots, K\}$ denotes the sensitive attribute.

We denote an encoder network  by $E_\w(\cdot)$ parameterized by $\w$,  
and let $E_{\w}(\x)\in\R^d$ represent a normalized output representation of input data $\x$. For simplicity, we omit the parameters and use $E(\cdot)$ to refer to the encoder. $\P$ denotes a set of standard data augmentation operators generating various views of a given image~\cite{chen2020simple}, and $\A\sim \P$ is a random data augmentation operator. 
We use $p(X)$ to denote the probability distribution of a random variable $X$, and use $p(X|Y)$ to represent the conditional distribution of a random variable $X$ given  $Y$.


A state-of-the-art method of SSL is to optimize a contrastive loss~\cite{qi2020simple, chen2020simple}. 
A standard approach for defining a contrastive loss of image data is the following mini-batch based contrastive loss for each image $\x_i$ and two random augmentation operators $\A,\A'\sim \P$~\cite{chen2020simple}: 
\begin{equation*}\label{eqn:con}
\begin{aligned}
    &L_{\text{CL}}(\x_i, \A,\A', \B)
    = - \log\frac{\exp(E(\A(\x_i))^{\top}E(\A'(\x_i))/\tau)}{ \sum_{\tilde{\x}\in\B_i^-\cup \{\A'(\x_i)\}}\exp(E(\A(\x_i))^{\top}E(\tx)/\tau)},
\end{aligned}
\end{equation*}
where 
$\tau>0$ is called the temperature parameter, $\B\subset \D$ is a random mini-batch and $\B_i^-=\{\A(\x),\A'(\x)|\x\in\mathcal B\setminus\x_i\}$ denotes the set of all other samples in the mini-batch and their two random augmentations. However, this approach requires a very large size to achieve good performance~\cite{chen2020simple} due to a large optimization error with a small batch size~\cite{yuan2022provable}. 

To address this challenge, Yuan et al.~\cite{yuan2022provable} proposed to optimize a {\bf global contrastive loss (GCL)} based on advanced optimization techniques with rigorous convergence guarantee. 
We adopte the second variant of GCL defined in their work in order to derive a convergence guaratnee. 
A GCL for a given sample $\x_i$ and two augmentation operators $\A,\A'\sim \P$ can be defined as:
\begin{equation*}\label{eqn:con}
\begin{aligned}
    &L_{\text{GCL}}(\x_i, \A,\A', \D) = - \tau \log\frac{\exp(E(\A(\x_i))^{\top}E(\A'(\x_i))/\tau)}{\epsilon_0 + \E_\A \sum_{\tx\in\S_i^-}\exp(E(\A(\x_i))^{\top}E_1(\tx)/\tau)},
    \end{aligned}
\end{equation*}
where $\epsilon_0$ is a small constant,  $\S_i^-= \{\A(\x) | \A\in\P,  \x\in\D\backslash {\x_i}\}$ denotes the set of all data to be contrasted with $\x_i$, which can be constructed by including all other images except for $\x_i$ and their augmentations. 
Then, the averaged GCL becomes 
   $F_{\text{GCL}}(\w)= \E_{\x_i\sim \D,\A,\A'\sim \P}L_{\text{GCL}}(\x_i, \A,\A', \D)$. 
To facilitate the design of stochastic optimization, we cast the above loss into the following  form:
\begin{align}
    F_{\text{GCL}}(\w):=&f_1(\w) +\frac{1}{n}\sum_{i=1}^nf_2(g(\w;\x_i,\S_{i}^{-})) + c,
\end{align}
where  
\begin{align*}
 &f_1(\w)=\E_{\x\sim\D,\A,\A'\sim\P}[f_1(\w; \x_i, \A, \A')], \quad f_1(\w; \x_i, \A, \A')= - E(\A(\x_i))^{\top}E(\A'(\x_i)),\\
 &f_2(g)=\tau\log(\epsilon_0' + g), \quad g(\w, \x_i, \S_i^{-}) = \E_{\tx\sim\S_i^-}\E_{\A}\exp(E(\A(\x_i))^{\top}E_1(\tx)/\tau),
 \end{align*}
 and $\epsilon_0'=\epsilon_0/|\S_i^{-}|$ is a small constant and $c$ is a constant. 
\vspace*{-0.1in}
\section{The Formulation and Justification}
\vspace*{-0.15in}


Our method is built on a classical idea of adversarial training by solving a minimax zero-sum  game~\cite{10.5555/2946645.2946704}. The idea is to maximize the loss of predicting the sensitive attribute by a model based on learned representations while minimizing a certain loss of learning the encoder network.  To this end, we define a discriminator $D_{\w'}(\v):\mathbb R^{d}\rightarrow \R^K$ parameterized by $\w'$ that outputs the probabilities of different values of $a$ for an input data associated with an encoded representation vector $\v\in\R^d$. For example, a simple choice of $D_{\w'}$ could be  $[D_{\w'}(\v)]_k = \frac{\exp(\w_k'^\top \v)}{\sum_{l=1}^K\exp({\w'_l}^{\top} \v)}$. 

Then, we introduce a fairness-promoting regularizer:
\begin{equation*}\label{eqn:fair_reg}
\begin{aligned}
    \max_{\w'}F_{\text{fair}}(\w,\w') := \E_{(\x, a)\sim\D_a,\A\sim\P}\phi(\w,\w';\A(\x), a),
\end{aligned}
\end{equation*}
where  $\phi(\w,\w'; \A(\x),a)$ denotes the  log-likelihood of the discriminator on predicting the sensitive attribute $a$ of the augmented data $\A(\x)$ based on $E_\w(\A(\x))$, i.e., $\phi(\w,\w'; \A(\x), a):= \log([D_{\w'}(E_{\w}(\A(\x)))]_a)$. 
Thus, the {\bf minimax zero-sum game} for learning the encoder network and the discriminator is imposed by:
\begin{equation}\label{eqn:fair_minmax}
\begin{aligned}
    \min_\w\max_{\w'}F(\w, \w'):=F_{\text{GCL}}(\w) + \alpha  F_{\fair}(\w,\w'),
\end{aligned}
\end{equation}
where $\alpha$ is a parameter that controls the trade-off between the GCL and the fairness regularizer.  There are several benefits of this approach compared with existing fairness-aware contrastive learning~\cite{park2022fair,10.1145/3543507.3583480,DBLP:journals/corr/abs-2106-02866,zhang2023fairnessaware}. First, the contrastive loss remains intact. Hence, it offers more flexibility of choosing different contrastive losses or even other losses for SSL and also makes it possible to extend our framework to multi-modal SSL~\cite{radford2021learning}. Second, only the fairness-promoting regularizer depends on the sensitive attribute. Therefore, it is not necessary to annotate all unlabeled data with sensitive attribute, which makes it more suitable for SSL with a large number of unlabeled data. 

Next, we present a theoretical justification of the minimax framework with a fairness-promoting regularizer. 
To formally quantify the fairness of learned representations, we define a distributional representation fairness as following.
\vspace*{-0.03in}\begin{defi}[{\bf Distributional Representation Fairness}]
For any random data $(\x, a)$, an encoder network $E$ is called fair in terms of representation if $p(E(\x)|a=k) = p(E(\x))$.
\end{defi}
\vspace*{-0.1in}
The above definition of distributional representation fairness resembles the demographic parity (DP) requiring $p(h(\x)|a=k) = p(h(\x))$, where $h(\x)$ is a predictor that predicts the class label for $\x$. However, distributional representation fairness is a stronger condition, which implies DP for downstream classification tasks if the encoder network is fixed. This is a simple argument. Let $\S$ denote a random labeled training set for learning a linear classification model $h(\x) = \w^{\top}E(\x)$ based on the encoded representation. Denote by $\w(\S)$ a learned predictive model.  With the distributional representation fairness, we have $p(E(\x)^{\top}\w(\S)|a=k) = p(E(\x)^{\top}\w(\S)) $ for a random data $(\x, a)$ that is independent of  a random labeled training set $\S$. This is due to that $E(\x)^{\top}\w(\S)|a=k$ and $E(\x)^{\top}\w(\S)$ have the same moment generating functions. Hence, DP holds, i.e., $p(\text{sign}(E(\x)^{\top}\w(\S))|a=k) = p(\text{sign}(E(\x)^{\top}\w(\S))) $.

Next, we present a proposition for justifying the zero-sum game framework. To this end, we abuse a notation $F_{\text{fair}}(E, D) = \E_{\x, a}\phi(E, D;\x, a)$ for any encoder $E$ and any discriminator $D$. 
\vspace*{-0.05in}\begin{prop}
\label{thm:fair_verification}
Suppose $E$ and $D$ have enough capacity, then the global optimal solution to the zero-sum game $\min_{E}\max_{D}F_{\fair}(E, D)$ denoted by $E_*, D_*$  would satisfy $p(E_*(\x)|a) = p(E_*(\x))$ and $[D_*(E_*(\x))]_k = p(a=k)$.
\end{prop}
\vspace*{-0.1in}
{\bf Remark:} We attribute the credit of the above result to earlier works, e.g.,~\cite{DBLP:journals/corr/abs-1906-08386}. It indicates the distribution of encoded representations is independent of the sensitive attribute. The proof of the above theorem is similar to the analysis of GAN. 

It is notable that the above result is different from that derived in~\cite{10.5555/3294771.3294827,DBLP:journals/corr/abs-1904-05514},  which considers a fixed encoder $E$ and only ensures the learned model recovers the true conditional distributions of the label and sensitive attribute given the representation. It has nothing related to fairness. 

\vspace*{-0.15in}
\section{Stochastic Algorithm and Analysis}
\vspace*{-0.15in}
The optimization problem~(\ref{eqn:fair_minmax}) deviates from existing studies of fair representation leaning in that (i) the GCL is a compositional function that requires more advanced optimization techniques in order to ensure convergence without requiring a large batch size~\cite{yuan2022provable}; (ii) the problem is a non-convex non-concave minimax compositional optimization, which is not a standard minimax optimization. 

\vspace*{-0.05in}\subsection{Algorithm Design}\vspace*{-0.05in}
A major challenge lies at how to compute a stochastic gradient estimator of $F_{\GCL}(\w)$. In particular, the term $\frac{1}{n}\sum_{i=1}^nf_2(g(\w;\x_i,\S_{i}^{-}))$ in the GCL is a finite-sum coupled compositional function~\cite{pmlr-v162-wang22ak}. The gradient $\nabla f_2(g(\w;\x_i,\S_{i}^{-}))\nabla g(\w;\x_i,\S_{i}^{-})$ is not easy to compute as the inner function  depends on a large number of data in $\S_i^{-}$. Because $f_2$ is non-linear, an unbiased stochastic estimation for the gradient $\nabla F_{\GCL}(\w)$ is not easily accessible by using mini-batch samples. In particular, the standard minibatch-based approach that uses $\nabla f_2(g(\w;\x_i,\B_{i}^{-}))\nabla g(\w;\x_i,\B_{i}^{-})$ as a gradient estimator for each sampled data $\x_i$ will suffer a large optimization error due to this estimator is not unbiased. 

\setlength{\floatsep}{10pt}
\setlength{\textfloatsep}{10pt}
\setlength{\intextsep}{10pt}
\begin{algorithm}[t]
    \centering
    \caption{Stochastic Optimization for Fair Contrastive Learning (SoFCLR)}\label{alg:NCSF}
    \begin{algorithmic}[1]
    \STATE \textbf{Initialization:} $\w_1,\w'_1,\u_1, \tilde{\m}_1$.
    \FOR{$t=1\cdots T$}
    \STATE Sample batches of data $\B\subset\D$, $\B_a\subset \D_a$.
    \FOR{$\x_i\in\B$}
    \item Sample two data augmentations $\A, \A'\sim\P$
    \STATE Calculate $g(\w, \A(\x_i),\B_i^-)$, $g(\w, \A'(\x_i),\B_i^-)$
        \STATE Update $\u_i^t$ according to~(\ref{eqn:u})
    \ENDFOR
    \STATE Compute $\m_{t+1}$  as in~(\ref{eqn:m}) and $\v_{t+1}$ as in~(\ref{eqn:v}).
    \STATE Compute $\widetilde{\m}_{t+1} =  (1-\beta)\widetilde{\m}_t + \beta \m_{t+1}$,
    \STATE Update  $\w_{t+1} = \w_t -\eta\widetilde{\m}_{t+1}$ (or Adam udpate)
    \STATE Update  $\w'$: $\w'_{t+1} = \w'_t + \eta' \v_{t+1}$
    \ENDFOR
    \end{algorithmic}
\end{algorithm}

To handle this challenge,  we will follow the idea in SogCLR~\cite{yuan2022provable} by maintaining and updating moving average estimators of the inner function values. Let us define $ g(\w; \A(\x_i), \tx)=\exp(E(\w;\A(\x_i))^\top E(\w; \tx)/\tau)$. Then $\E_{\A, \tx}[g(\w, \A(\x_i), \tx)] = g(\w; \x_i, \S_i^-)$.
We use the vector $\u = [\u_1,\cdots,\u_n]\in \R^n$ to track the moving average history of stochastic estimator of $g(\w;\x_i, \S^-_i)$. For sampled $\x_i\in \B$, we update $\u_{i,t}$ by
\begin{equation}
\label{eqn:u}
    \begin{aligned}
        &\u_{i, t+1} = (1-\gamma)\u_{i, t} + \frac{\gamma}{2}\left[ g(\w_t;\A(\x_i), \B^-_i)+g(\w_t;\A'(\x_i), \B^-_i)\right],
    \end{aligned}
\end{equation}
where $\gamma \in (0,1)$ denotes the moving average parameter. For unsampled $\x_i\not\in \B$, no update is needed, i.e., $\u_{i, t+1} = \u_{i, t}$. Then,  $\nabla f_2(g(\w_t;\x_i, \S_{i}^{-}))\nabla g(\w_t;\x_i,\S_{i}^{-})$ can be estimated by  $\nabla f_2(\u_{i,t})\nabla g(\w_t;\x_i,\B_{i}^{-})$. Compared to the simple minibatch estimator, this estimator ensures diminishing average error. 
Thus, we compute a stochastic gradient estimator of $\nabla_\w F(\w_t,\w_t')$  by
    \begin{align}
      &\m_{t+1}\label{eqn:m}= \frac{1}{|\B|}\sum\limits_{\x_i\in\B}\bigg\{
   \nabla_{\w} f_1(\w_t;\x_i, \A, \A' ) + \frac{\tau (\nabla_\w g(\w_t; \A(\x_i), \B_i^-) + \nabla_\w g(\w_t;\A'(\x_i),\B_i^-))}{2(\epsilon_0' + \u_{i,t})}\bigg\}\notag\\
& +\frac{\alpha}{2|\B_a|}\sum\limits_{\x_i\in\B_a}
\bigg\{\nabla_\w\phi(\w_t,\w'_t;\A(\x_i),a_i)+\nabla_\w\phi(\w_t,\w'_t;\A'(\x_i),a_i)\bigg\}.
    \end{align}
Then we update the primal variable $\w_{t+1}$ using either the momentum method or the Adam method. For updating the dual variable $\w'$, we can employ stochastic gradient ascent-type update  based on the following stochastic gradient: 
\begin{equation}
\label{eqn:v}
    \begin{aligned}
        \v_{t+1} & = \frac{1}{2|\B_a|}\sum\limits_{\x_i\in\B_a} (\nabla_{\w'}\phi(\w_t,\w'_t;\A(\x_i),a_i)+\nabla_{\w'}\phi(\w_t,\w'_t;\A'(\x_i),a_i)).
    \end{aligned}
\end{equation}
Finally, we present detailed steps of the proposed stochastic algorithm in Algorithm~\ref{alg:NCSF}, which is referred to as SoFCLR. For simplicity of exposition, we use stochastic momentum update for the primal variable $\w$ and the stochastic gradient ascent update for the dual variable $\w'$. 

\vspace*{-0.05in}\subsection{Convergence Analysis}\vspace*{-0.05in}
The convergence analysis is complicated by the presence of non-convexity and non-concavity of the minimax structure and coupled compositional structure. Our goal is to derive a convergence for finding an $\epsilon$-stationary point to the primal objective function $\Phi(\w) = \max_{\w'}F(\w,\w')$, i.e., a point $\w$ such that $\E[\|\nabla\Phi(\w)\|^2]\leq \epsilon^2$.  We emphasize that it is generally impossible to prove this result without imposing some conditions of the objective function.



\begin{ass}\label{ass:1}
    We make the following assumptions.
    \vspace*{-0.1in}
    \begin{enumerate}[label=(\alph*)]
        \item $\Phi(\cdot)$ is $L$-smooth.
        \vspace{-0.05in}
        \item $E_\w(\x)$ is smooth and Lipchitz continuous w.r.t $\w$. There exists constant $C_E<\infty$ such that $\|E_\w(\x)\|\leq C_E$ for all $\w$ and $\x$.
                \vspace{-0.05in}
        \item $D_{\w'}(\v)$ is smooth and Lipchitz continuous w.r.t $\w'$. 
                \vspace{-0.05in}
        \item There exists $\Delta$ such that $\Phi(\w_1)-\min_{\w}\Phi(\w)\leq \Delta$.
                \vspace{-0.05in}
        \item  The following variances are bounded: 
        \begin{align*}
        &\E_{\A, \widetilde\x}(g(\w; \A(\x_i), \widetilde\x) - g(\w; \x_i, \S_i^{-}))^2\leq \sigma^2\\
        &\E_{\x, \A,\A'}\|\nabla_{\w} f_1(\w;\x,\A,\A') - \nabla f_1(\w)\|^2\leq \sigma^2\\
        &\E_{\A, \widetilde\x}\|\nabla_\w g(\w; \A(\x_i), \widetilde\x) - \nabla_{\w}g(\w; \x_i, \S_i^{-}) \|^2\leq \sigma^2,\\
        &\E_{\A, \x, a}\|\nabla_\w\phi(\w,\w';\A(\x),a) - \nabla_{\w}F_{\text{fair}}(\w, \w')\|^2\leq \sigma^2,\\
        &\E_{\A, \x, a}\|\nabla_{\w'}\phi(\w,\w';\A(\x),a) - \nabla_{\w'}F_{\text{fair}}(\w, \w')\|^2\leq \sigma^2.
        \end{align*}
           \item For any $\w$, $F_{\text{fair}}(\w,\cdot)$ satisfies \begin{equation}\label{eqn:oq}-(\w'_*  - \w' )^{\top}\nabla_{\w'}F_{\text{fair}}(\w,\w')\geq \lambda \|\w' - \w'_*\|^2,\end{equation} 
        where $\w'_* \in \arg\max_{\v} F_{\text{fair}}(\w,\v)$ is one optimal solution  closest to $\w'$.
    \end{enumerate}
\end{ass}
\vspace*{-0.1in}
{\bf Remark:} Conditions (a, b, c) are simplifications to ensure the objective function and each component function are smooth. 
Conditions (d, e) are standard conditions for stochastic non-convex optimization. Condition (f) is a special condition, which is called one-point strong convexity~\cite{DBLP:conf/nips/Yuan0JY19} or restricted secant inequality~\cite{DBLP:journals/corr/KarimiNS16}.  It does not necessarily require the convexity in terms of $\w'$ and much weaker than strong convexity~\cite{DBLP:journals/corr/KarimiNS16}.  It has been proved for wide neural networks~\cite{DBLP:journals/corr/abs-2306-02601}, i.e., when $D_{\w'}$ is a wide neural network.


\begin{thm}\label{thm:1}
Under the above assumption and  parameter setting $\eta= \O\left(\min\left\{\beta,\frac{|\B|\gamma}{ n}, \eta'\right\}\right)$, $\beta=\O\left(\min\{|\B|,|\B_a|\}\epsilon^2\right)$, $\gamma=\O\left( |\B|\epsilon^2 \right)$, $\eta'= \O\left(\lambda |\B_a|\epsilon^2\right)$, after $T= 
     \O\Big(\max$ $\left\{\frac{1}{\min\{|\B|,|\B_a|\}},\frac{n}{|\B|^2},\frac{1}{\lambda^3|\B_a|}\right\}\epsilon^{-4}\Big )$
iterations, SoFCLR can find an $\epsilon$-stationary solution of $\Phi(\cdot)$.
\end{thm}
\vspace*{-0.1in}
{\bf Remark:} We can see that $\epsilon^{-4}$ matches the complexity of SGD for non-convex minimization problems. In addition, the factor $n/|\B|^2$ is the same as that of SogCLR for optimizing the GCL in~\cite{yuan2022provable}. The additional factor $1/(\lambda^3|\B_a|)$ is due to the maximization of a non-concave function.  We refer the detailed statement and proof of Theorem~\ref{thm:1} to Appendix~\ref{apd:analysis}.

\vspace*{-0.15in}
\section{Experiments}\vspace*{-0.1in}

\noindent \textbf{Datasets.} We use two face image datasets for our experiments, namely CelebA~\cite{liu2015faceattributes} and UTKface~\cite{zhifei2017cvpr}. CelebA is a large-scale face attributes dataset with more than 200K celebrity images, each with binary annotations of 40 attributes. UTKFace includes more than 20K face images labeled by  gender, age, and ethnicity.  These two datasets have been used in earlier works of fair representation learning~\cite{park2022fair,DBLP:journals/corr/abs-2106-02866,zhang2023fairnessaware}. We will construct binary classification tasks on both datasets as detailed later.  

\noindent \textbf{Methodology of Evaluations.} We will evaluate our algorithms from two perspectives: (i) quantitative performance on downstream classification tasks; (ii) qualitative visualization of learned representations.  For quantitative evaluation, we first perform SSL by our algorithm on an unlabeled dataset with partial sensitive attribute information. Then we utilize a labeled training dataset for learning a linear classifier based on the learned representations, and then evaluate the accuracy and fairness metrics on testing data. This approach is known as linear evaluation in the literature of SSL~\cite{chen2020simple}. 

\begin{table*}[h!]
     \caption{Results on CelebA: accuracy of predicting Attractive and fairness metrics for two sensitive attributes, Male and Young.
     }
    \label{tab:celeba_ncsf_attractive}   \centering
    \resizebox{0.99\textwidth}{!}{
    \begin{tabular}{c|c|c|c|c|c|c|c|c|c} \toprule
   (Attractive, Male)                 & Acc                  & $\Delta$ ED                                 & $\Delta$ EO                                  & $\Delta$ DP                                  & IntraAUC                                      & InterAUC                                      & GAUC                                         & WD                                            & KL                                            \\
     \toprule
CE                                 & 80.20 ($\pm$ 0.31)   & 25.55 ($\pm$ 0.27)                          & 22.53 ($\pm$ 0.47)                           & 45.40 ($\pm$ 0.56)                           & 0.0024 ($\pm$ 1e-3)                           & 0.2745 ($\pm$ 3e-3)                           & 0.3053 ($\pm$ 3e-3)                          & 0.3131 ($\pm$ 3e-3)                           & 0.7153 ($\pm$ 4e-3)                           \\
CE + EOD                           & 79.70 ($\pm$ 0.41)   & 22.18 ($\pm$ 0.31)                          & 16.75 ($\pm$ 0.28)                           & 41.65 ($\pm$ 0.44)                           & 0.0014 ($\pm$ 1e-3)                           & 0.2372 ($\pm$ 4e-3)                           & 0.2897 ($\pm$ 2e-3)                          & 0.2804 ($\pm$ 4e-3)                           & 0.6189 ($\pm$ 5e-3)                           \\
CE + DPR                           & 80.08 ($\pm$ 0.28)   & 23.74 ($\pm$ 0.48)                          & 17.15 ($\pm$ 0.21)                           & 43.06 ($\pm$ 0.34)                           & 0.0051 ($\pm$ 5e-4)                           & 0.2571 ($\pm$ 3e-3)                           & 0.2981 ($\pm$ 3e-3)                          & 0.2924 ($\pm$ 5e-3)                           & 0.6761 ($\pm$ 4e-3)                           \\
CE + EQL                           & 79.63 ($\pm$ 0.29)   & 25.10 ($\pm$ 0.36)                          & 20.10 ($\pm$ 0.35)                           & 44.50 ($\pm$ 0.38)                           & \textbf{0.0024} ($\pm$ 4e-4) & 0.2738 ($\pm$ 4e-3)                           & 0.3037 ($\pm$ 4e-3)                          & 0.2975 ($\pm$ 3e-3)                           & 0.7177 ($\pm$ 4e-3)                           \\
ML-AFL                             & 79.44($\pm$ 0.32)    & 32.12 ($\pm$ 0.33)                          & 23.39 ($\pm$ 0.41)                           & 48.70 ($\pm$ 0.35)                           & 0.0030 ($\pm$ 8e-4)                           & 0.3561 ($\pm$ 5e-3)                           & 0.3382 ($\pm$ 3e-3)                          & 0.3341 ($\pm$ 3e-3)                           & 0.9551 ($\pm$ 3e-3)                           \\
Max-Ent                            & 79.46 ($\pm$ 0.28)   & 30.72 ($\pm$ 0.29)                          & 18.42 ($\pm$ 0.38)                           & 47.42 ($\pm$ 0.40)                           & 0.0046 ($\pm$ 2e-3)                           & 0.3241 ($\pm$ 4e-3)                           & 0.3289 ($\pm$ 5e-3)                          & 0.3083 ($\pm$ 4e-3)                           & 0.9215 ($\pm$ 3e-3)                           \\ \midrule
SimCLR     & 80.11 ($\pm$ 0.28)   & 26.58 ($\pm$ 0.34)                          & 17.34 ($\pm$ 0.38)                           & 44.95 ($\pm$ 0.32)                           & 0.0055 ($\pm$ 1e-3)                           & 0.2835 ($\pm$ 5e-3)                           & 0.3211 ($\pm$ 4e-3)                          & 0.2458 ($\pm$ 4e-3)                           & 0.8276 ($\pm$ 4e-3)                           \\
{SogCLR}     & 80.53 ($\pm$ 0.25)   & 25.38 ($\pm$ 0.28)                          & 18.71 ($\pm$ 0.33)                           & 44.51 ($\pm$ 0.31)                           & 0.0035 ($\pm$ 6e-4)                           & 0.2659 ($\pm$ 4e-3)                           & 0.3167 ($\pm$ 3e-3)                          & 0.2432 ($\pm$ 3e-3)                           & 0.8055 ($\pm$ 4e-3)                           \\
Boyl       & 79.58 ($\pm$ 0.29)   & 24.51 ($\pm$ 0.41)                          & 20.99 ($\pm$ 0.37)                           & 47.02 ($\pm$ 0.28)                           & 0.0091 ($\pm$ 5e-4)                           & 0.2713 ($\pm$ 7e-3)                           & 0.3974 ($\pm$ 5e-3)                          & 0.2367 ($\pm$ 5e-3)                           & 0.7641 ($\pm$ 5e-3)                           \\
SimCLR+CCL & 79.91 ($\pm$ 0.28)   & 22.19 ($\pm$ 0.35)                          & 18.59 ($\pm$ 0.32)                           & 39.58 ($\pm$ 0.30)                           & 0.0069 ($\pm$ 6e-4)                           & 0.3146 ($\pm$ 5e-3)                           & 0.3059 ($\pm$ 4e-3)                          & 0.2143 ($\pm$ 4e-3)                           & 0.6408 ($\pm$ 5e-3)                           \\ \midrule
SoFCLR     & 79.95 ($\pm$ 0.19)   & \textbf{14.93}($\pm$ 0.22) &\textbf{12.60} ($\pm$ 0.25) & \textbf{36.50} ($\pm$ 0.24) & 0.0032 ($\pm$ 3e-4)                           & \textbf{0.1592} ($\pm$ 2e-3) & \textbf{0.2566}($\pm$ 2e-3) & \textbf{0.1402} ($\pm$ 2e-3) & \textbf{0.4743} ($\pm$ 4e-3)\\
        \toprule
(Attractive, Young) & Acc                  & $\Delta$ ED                                  & $\Delta$ EO                                  & $\Delta$DP                                   & IntraAUC                                      & InterAUC                                      & GAUC                                          & WD                                            & KL                                           \\ 
\midrule
CE                  & 79.23 ($\pm$ 0.33)   & 22.79 ($\pm$ 0.31)                           & 15.90 ($\pm$ 0.32)                           & 40.47 ($\pm$ 0.34)                           & 0.0358 ($\pm$ 2e-3)                           & 0.2434 ($\pm$ 4e-4)                           & 0.3129 ($\pm$ 3e-3)                           & 0.3047 ($\pm$ 3e-3)                           & 0.7275 ($\pm$ 4e-3)                          \\
CE + EOD            & 79.03 ($\pm$ 0.34)   & 22.78 ($\pm$ 0.28)                           & 15.69 ($\pm$ 0.35)                           & 41.81 ($\pm$ 0.38)                           & 0.0403 ($\pm$ 3e-3)                           & 0.2409 ($\pm$ 3e-4)                           & 0.3142 ($\pm$ 4e-3)                           & 0.3079 ($\pm$ 2e-3)                           & 0.7434 ($\pm$ 3e-3)                          \\
CE + DPR            & 78.51 ($\pm$ 0.29)   & 22.08 ($\pm$ 0.30)                           & 15.17 ($\pm$ 0.29)                           & 40.83 ($\pm$ 0.41)                           & 0.0396 ($\pm$ 2e-3)                           & 0.2368 ($\pm$ 3e-4)                           & 0.3110 ($\pm$ 3e-3)                           & 0.3039 ($\pm$ 3e-3)                           & 0.7165 ($\pm$ 3e-3)                          \\
CE + EQL            & 80.02 ($\pm$ 0.30)   & 22.09 ($\pm$ 0.34)                           & 15.68 ($\pm$ 0.33)                           & 41.53 ($\pm$ 0.35)                           & 0.0390 ($\pm$ 2e-3)                           & 0.2332 ($\pm$ 5e-4)                           & 0.3095 ($\pm$ 4e-3)                           & 0.3020 ($\pm$ 4e-3)                           & 0.7082 ($\pm$ 4e-3)                          \\
ML-AFL              & 79.25 ($\pm$ 0.31)   & 31.97 ($\pm$ 0.31)                           & 22.50 ($\pm$ 0.31)                           & 48.70 ($\pm$ 0.29)                           & 0.0451 ($\pm$ 3e-3)                           & 0.3560 ($\pm$ 4e-4)                           & 0.3380 ($\pm$ 4e-3)                           & 0.3340 ($\pm$ 3e-3)                           & 0.9551 ($\pm$ 3e-3)                          \\
MaxEnt-ALR          & 79.33  ($\pm$ 0.30 ) & 30.59 ($\pm$ 0.29)                           & 18.10 ($\pm$ 0.30)                           & 46.99 ($\pm$ 0.34)                           & 0.0420 ($\pm$ 4e-3)                           & 0.2113 ($\pm$ 5e-4)                           & 0.3117 ($\pm$ 3e-3)                           & 0.2927 ($\pm$ 4e-3)                           & 0.7285 ($\pm$ 3e-3)                          \\ \midrule
SimCLR              & 79.97 ($\pm$ 0.29)   & 17.52 ($\pm$ 0.31)                           & 18.50 ($\pm$ 0.31)                           & 42.47 ($\pm$ 0.41)                           & 0.0381 ($\pm$ 3e-3)                           & 0.1909 ($\pm$ 4e-4)                           & 0.2984 ($\pm$ 3e-3)                           & 0.2098 ($\pm$ 3e-3)                           & 0.6877 ($\pm$ 2e-3)                          \\
SogCLR              & 79.73 ($\pm$ 0.27)   & 17.21 ($\pm$ 0.28)                           & 17.61 ($\pm$ 0.27)                           & 42.01 ($\pm$ 0.32)                           & 0.0365  ($\pm$ 3e-3)                          & 0.1940 ($\pm$ 4e-4)                           & 0.3021 ($\pm$ 4e-3 )                          & 0.2114 ($\pm$ 4e-3)                           & 0.6782 ($\pm$ 3e-3)                          \\
Boyl                & 79.83 ($\pm$ 0.28)   & 17.03($\pm$ 0.31)                            & 18.03 ($\pm$ 0.29)                           & 43.01 ($\pm$ 0.29)                           & 0.0393  ($\pm$ 4e-3)                          & 0.2001 ($\pm$ 8e-4)                           & 0.3233 ($\pm$ 3e-3 )                          & 0.2214 ($\pm$ 5e-3)                           & 0.6804 ($\pm$ 4e-3)                          \\
SimCLR+CCL        & 79.87 ($\pm$ 0.26)              & 17.18 ($\pm$ 0.31)                                       & 17.25( $\pm$ 0.28)                           & 42.05  ($\pm$ 0.30)                          & 0.0385  ($\pm$ 3e-3)                          & 0.1891  ($\pm$ 5e-4)                          & 0.2824  ($\pm$ 4e-3 )                         & 0.2158 ($\pm$ 5e-3)                           & 0.6532 ($\pm$ 4e-3)                          \\ \midrule
SoFCLR              & 79.93 ($\pm$ 0.25)   & \textbf{15.34} ($\pm$ 0.27) & \textbf{14.10} ($\pm$ 0.25) & \textbf{40.05} ($\pm$ 0.26) & \textbf{0.0336} ($\pm$ 2e-3) & \textbf{0.1652} ($\pm$ 3e-4) & \textbf{0.2824} ($\pm$ 2e-3) & \textbf{0.1506} ($\pm$ 3e-3) & \textbf{0.5905}($\pm$ 2e-3) \\
        \bottomrule
    \end{tabular}}
\end{table*}
\noindent \textbf{Baselines.}  We compare with 10 baselines from different categories, including fairness-unaware SSL methods, fairness-aware SSL methods, and conventional fairness-aware supervised learning methods.  Fairness-unaware SSL methods include SimCLR~\cite{chen2020simple}, SogCLR~\cite{yuan2022provable}, and Byol~\cite{DBLP:journals/corr/abs-2006-07733}. SimCLR and SogCLR are contrastive learning methods and Byol is non-contrastive SSL methods. For fairness-aware SSL, most existing approaches either assume all data have sensitive attribute information~\cite{ma2021conditional} or rely on an image generator that is trained separately~\cite{zhang2023fairnessaware}. In order to be fair with our algorihtm, we construct a strong baseline by combining the loss of SimCLR and CCL~\cite{ma2021conditional}, where a mini-batch contrastive loss is defined on unlabled data, and a conditional contrastive loss is defined on unlabeled data with sensitive attribute information. 
We refer to baseline as SimCLR+CCL. For fairness-aware supervised learning, we consider two adversarial fair representation learning approaches namely Max-Ent~\citep{roy2019mitigating} and Max-AFL proposed~\cite{xie2017controllable}, and three fairness regularized approaches that optimize the cross-entropy (CE) loss and a fairness regularizer, including equalized odds regularizer (EOD), demographic disparity regularizer (DPR), equalized loss regularizer (EQL). These methods have been considered in previous works~\cite{cherepanova2021technical,donini2020empirical}. 
We also include a reference method which just minimizes the CE loss on labeled data. All the experiments are conducted on a server with four GTX1080Ti GPUs. 

\noindent \textbf{Fairness Metrics.} In order to evaluate the effectiveness of our algorithm, we evaluate a total of 8 fairness metrics. These include commonly used demographic disparity ($\Delta$ DP), equalized odds difference ($\Delta$ ED), and equalized opportunity ($\Delta$ EO),  three AUC fairness metrics namely  group AUC fairness (GAUC)~\cite{yao2023stochastic}, Inter-group AUC fairness (InterAUC) and Intra-group AUC fairness (IntraAUC)~\cite{beutel2019fairness}, and two distance metrics that measure the distance between the distributions of prediction scores of examples from different groups. To this end,  we discretize the prediction scores on testing examples from each group into 100 buckets and calculate the KL-divergence and Wasserstein distance (WD) of two empirical distributions of two groups. 
For fairness metrics, smaller values indicate better performance.

\noindent{\bf Neural networks and optimizers.}  We utilize ResNet18 as the backbone network. For SSL, we add  a two-layer MLP for the projection head  with widths $256$ and $128$. For our algorithm, we use a two-layer MLP for predicting the sensitive attribute with a hidden dimension of $512$. The updates of model parameters follow the Adam update. The detailed information  is described in Appendix~\ref{sec:expm}. 



\noindent \textbf{Hyperparameter tuning.} Each method has some hyper-parameters including those in objective function and optimizers, e.g., the combination weight of a regular loss and fairness regularizer, the learning rate of optimizers, please check in Appendix~\ref{sec:expm} for tuning details. In addition, there are multiple fairness metrics besides the accuracay performance. For hyperparameter tuning, we divide the data into training, validation and testing sets. The validation and testing sets have target labels and sensitive attributes. For each method, we tune their parameters in order to obtain an accuracy within a certain threshold ($1\%$) of the standard CE baseline, and then report their different fairness metrics.   

\vspace*{-0.2in}
\subsection{Prediction Results}
\label{sec:experiemnt-prediction}
\vspace*{-0.1in}
\noindent{\bf Results on CelebA.} Following the same setting as~\cite{park2022fair}, we use two attributes `Male' and `Young' to define the sensitive attribute. Each attribute divides data into two groups according the binary annotation of each attribute. For target labels, we considier three attributes that have the highest Pearson correlation with the sensitive attribute, i.e, Attractive, Big\  Nose, and Bags\ Under\ Eyes. Hence, we have six tasks in total. Due to limit of space, we only report the results for two tasks,  (Attractive, Male), and (Attractive, Young), and include more results in the Appendix~\ref{sec:more-expri} for other tasks. We use the {80\%/10\%/10\%} splits for constructing the training, validation and testing datsets. We assume a subset of 5\% random training examples have sensitive attribute information for training by SoFCLR, SimCLR + CCL, and supervised baseline methods. The supervised  methods just use those 5\% images including their target attribute labels and sensitive attributes for learning a model.  

\begin{figure*}
\centering
\includegraphics[width=0.9\linewidth]{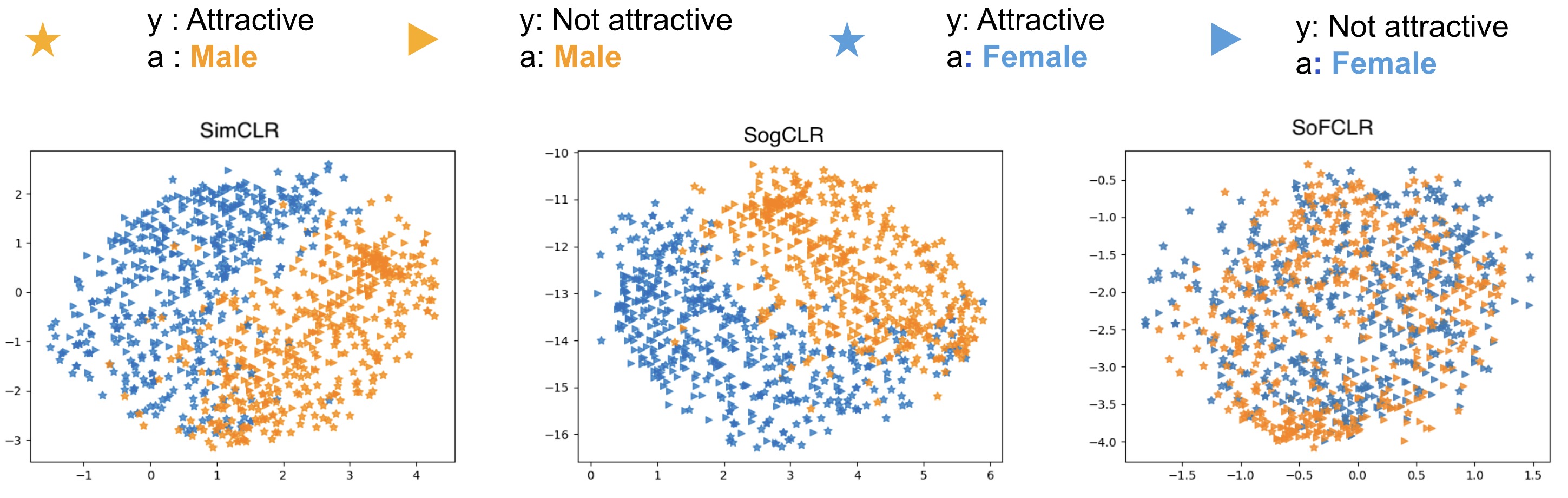}
\vspace*{-0.1in}      \caption{Learned representations of 1000 testing examples from CelebA by different methods. 
}\vspace*{-0.1in}  \label{fig:TSNE_CelebA}
\end{figure*}
The results are presented in Table~\ref{tab:celeba_ncsf_attractive} for CelebA. From the results, we can observe that: (1) our method SoFCLR yields much fair results than existing fairness unaware SSL methods SimCLR, SogCLR and Byol; (2) compared with SimCLR-CCL, we can see that SoFCLR is more effective for obtaining more fair results; (3) the fairness-aware supervised methods are not that effective in our considered setting. This is probably because they only use 5\% labeled data for learning the prediction model.

\begin{table*}[h!]
    \caption{Results on UTKFace: accuracy of predicting gender and fairness metrics in terms of two sensitive attributes, Age and Ethinicity. }
        \label{tab:utkface_self_clr}    \centering
    \resizebox{\textwidth}{!}{
    \begin{tabular}{c|c|c|c|c|c|c|c|c|c} \toprule
   (Gender, Age) & Acc                & $\Delta$ ED        & $\Delta$ EO        & $\Delta$ DP        & IntraAUC            & InterAUC            & GAUC                & WD                  & KL \\ \midrule
SimCLR                   & 85.74 ($\pm$ 0.31) & 19.60 ($\pm$ 0.34) & 28.32 ($\pm$ 0.41) & 19.62 ($\pm$ 0.41) & {\bf 0.0457 ($\pm$ 3e-4)} & 0.1287 ($\pm$ 4e-3) & 0.1156 ($\pm$ 4e-3) & 0.1512 ($\pm$ 3e-3) & 0.1368 ($\pm$ 4e-3) \\
SogCLR                   & 85.86 ($\pm$ 0.34) & 17.83 ($\pm$ 0.32) & 28.28 ($\pm$ 0.29) & 17.58 ($\pm$ 0.31) & 0.0471 ($\pm$ 4e-4) & 0.1227 ($\pm$ 5e-3) & 0.1145 ($\pm$ 3e-3) & 0.1458 ($\pm$ 4e-3) & 0.1416 ($\pm$ 3e-3)     \\
Byol                     & 85.37 ($\pm$ 0.37) & 17.97($\pm$ 0.36)  &28.37($\pm$ 0.25) & 17.49 ($\pm$ 0.29) & 0.0496 ($\pm$ 5e-4) & 0.1221 ($\pm$ 4e-3) & 0.1132 ($\pm$ 2e-3) & 0.1467 ($\pm$ 5e-3) & 0.1383 ($\pm$ 3e-3)    \\
SimCLR+CCL               & 85.56 ($\pm$ 0.36) & 16.83 ($\pm$ 0.33) & 27.32 ($\pm$ 0.25) & 17.08 ($\pm$ 0.28) & 0.0483 ($\pm$ 4e-4) & 0.1203 ($\pm$ 4e-3) & 0.1098 ($\pm$ 3e-3) & 0.1329 ($\pm$ 4e-3) & 0.1374 ($\pm$ 3e-3)   \\
SoFCLR                   & 85.89 ($\pm$ 0.27) & {\bf 15.42 ($\pm$ 0.28)} & {\bf 25.00 ($\pm$ 0.25)} & {\bf 15.49 ($\pm$ 0.26)} & { 0.0466 ($\pm$ 3e-4)} & {\bf 0.1041 ($\pm$ 4e-3)} & {\bf 0.0901 ($\pm$ 2e-3)} & {\bf 0.1151 ($\pm$ 3e-3)} & {\bf 0.1012 ($\pm$ 2e-3)} \\
     \bottomrule
(Gender, Ethnicity)  & Acc                & $\Delta$ ED        & $\Delta$ EO        & $\Delta$ DP        & IntraAUC            & InterAUC            & GAUC                & WD                  & KL                  \\ \midrule
SimCLR                         & 83.58 ($\pm $0.34) & 17.23 ($\pm$ 0.24) & 14.43 ($\pm$ 0.27) & 17.21 ($\pm$ 0.34) & 0.0091 ($\pm$ 6e-4) & 0.1352 ($\pm$ 6e-3) & 0.1375 ($\pm$ 5e-3) & 0.1591 ($\pm$ 4e-3) & 0.2017 ($\pm$ 5e-3) \\
SogCLR                         & 84.03 ($\pm$ 0.37) & 16.56 ($\pm$ 0.33) & 13.83 ($\pm$ 0.28) & 16.37 ($\pm$ 0.26) & 0.0083 ($\pm$ 5e-4) & 0.1284 ($\pm$ 5e-3) & 0.1572 ($\pm$ 4e-3) & 0.1353 ($\pm$ 5e-3) & 0.1913 ($\pm$ 6e-3) \\
Byol                           & 83.87 ($\pm$ 0.41) & 16.92 ($\pm$ 0.29) & 14.08 ($\pm$ 0.31) & 16.85 ($\pm$ 0.27) & 0.0087 ($\pm$ 6e-4) & 0.1273 ($\pm$ 6e-3) & 0.1523 ($\pm$ 4e-3) & 0.1195 ($\pm$ 6e-3) & 0.1237 ($\pm$ 5e-3) \\
SimCLR+CCL                     & 83.59 ($\pm$ 0.30) & 15.87 ($\pm$ 0.28) & 13.70 ($\pm$ 0.29) & 15.79 ($\pm$ 0.31) & {\bf 0.0081 ($\pm$ 4e-4)} & 0.1192 ($\pm$ 4e-3) & 0.1387 ($\pm$ 5e-3) & 0.1195 ($\pm$ 4e-3) & 0.1237 ($\pm$ 4e-3) \\
SoFCLR                         &  84.42 ($\pm$ 0.27) & {\bf 13.02 ($\pm$ 0.22)} & {\bf 13.23 ($\pm$ 0.24)} & {\bf 13.00 ($\pm$ 0.25)} & 0.0084 ($\pm$ 4e-4) & {\bf 0.1013 ($\pm$ 4e-3)} & {\bf 0.1029 ($\pm$ 3e-3)} & {\bf 0.1195 ($\pm$ 4e-3)} & {\bf 0.1237 ($\pm$ 3e-3)}\\
  \bottomrule
    \end{tabular}}
\vspace*{-0.1in}
    \centering
    \caption{Transfer learning results on UTKFace. SSL on CelebA, linear evaluation on UTKface with accuracy of predicting gender and fairness metrics in terms of age.  
}    \label{tab:cross-set-ups}
    \resizebox{0.98\textwidth}{!}{
    \begin{tabular}{c|c|c|c|c|c|c|c|c|c} \toprule 
      &Acc & $\Delta$ ED & $\Delta$ EO &  $\Delta$ DP &IntraAUC &InterAUC& GAUC & WD & KL \\
    \toprule
 CE         & 93.06 ($\pm$ 0.37) & 9.16 ($\pm$ 0.34) & 8.95 ($\pm$ 0.31) & 7.63 ($\pm$ 0.44) & 0.0375 ($\pm$ 4e-3) & 0.0245 ($\pm$ 4e-3) & 0.0451 ($\pm$ 5e-3) & 0.0684 ($\pm$ 4e-3) & 0.1356 ($\pm$ 4e-3) \\ \hline
SimCLR     & 89.04 ($\pm$ 0.41) & 6.89 ($\pm$ 0.35) & 8.54 ($\pm$ 0.38) & 5.91 ($\pm$ 0.48) & 0.0362 ($\pm$ 5e-3) & 0.0325 ($\pm$ 5e-3) & 0.0303 ($\pm$ 4e-3) & 0.0437 ($\pm$ 5e-3) & 0.0769 ($\pm$ 6e-3) \\
SogCLR     & 89.74 ($\pm$ 0.35) & 6.32 ($\pm$ 0.37) & 8.16 ($\pm$ 0.42) & 6.31 ($\pm$ 0.36) & 0.0296 ($\pm$ 4e-3) & 0.0321 ($\pm$ 4e-3) & 0.0298 ($\pm$ 4e-3) & 0.0436 ($\pm$ 3e-3) & 0.0743 ($\pm$ 5e-3) \\
Boyl       & 89.83 ($\pm$ 0.43) & 6.21 ($\pm$ 0.39) & 8.09 ($\pm$ 0.31) & 5.87 ($\pm$ 0.29) & 0.0303 ($\pm$ 3e-3) & 0.0331 ($\pm$ 7e-3) & 0.0301 ($\pm$ 3e-3) & 0.0478 ($\pm$ 4e-3) & 0.0723 ($\pm$ 7e-3) \\ 
SimCLR+CCL & 89.79 ($\pm$ 0.39) & 5.78 ($\pm$ 0.37) & 7.35 ($\pm$ 0.41) & 5.67 ($\pm$ 0.49) & 0.0292 ($\pm$ 4e-3) & 0.0197 ($\pm$ 5e-3) & 0.0283 ($\pm$ 4e-3) & 0.0391 ($\pm$ 4e-3) & 0.0683 ($\pm$ 6e-3) \\
SoFCLR     & 89.42 ($\pm$ 0.29) & {\bf 4.46 ($\pm$ 0.30)} & {\bf 5.89 ($\pm$ 0.23)} & {\bf 4.49 ($\pm$ 0.27)} & {\bf 0.0282 ($\pm$ 3e-3)} & {\bf 0.0176 ($\pm$ 3e-3)} & {\bf 0.0203 ($\pm$ 3e-3)} & {\bf 0.0271 ($\pm$ 3e-3)} & {\bf 0.0573 ($\pm$ 4e-3)} \\

        \bottomrule
\end{tabular}}
\vspace*{-0.1in}
\end{table*}

\vspace{-0.05in}
\noindent{\bf Results on UTKface.} 
There are three attributes for each image, i.e., gender, age, and ethnicity. In our experiments, we use gender as the target label, and the other two as the sensitive attribute.  To control the imbalance of different sensitive attribute group, following the setting of~\cite{park2022fair},  we manually construct imbalanced training sets in terms of the sensitive attribute that are highly correlate with the target label (e.g., Caucasian dominates the Male class)  
and keep the validation and testing data balanced. 
The details of training data statistics are described in Table~\ref{tab:enthnicity_gender}.

We consider two experimental settings. The first setting is to learn both the representation network and the classifier on the same data.    We split train/validation/test with 10000/3200/3200 images. However, since UTKface is a small dataset,   5\% of the labeled training samples totaling 500, is not sufficient to learn a good ResNet18. Hence, we just compare different SSL methods on this dataset. For SoFCL and SimCLR-CCL, we still assume  5\% of training data have sensitive attribute information. The linear evaluation is using all training images with target labels. The result shown in Table~\ref{tab:utkface_self_clr} indicate that  SoFCLR outperforms all baselines by a large margin on all fairness metrics.

The second setting is to learn a representation network on CelebA using SSL and perform linear evaluation on the UTK data. This is valid because both data have attribute information related to age, which is the Young attribute in CelebA and the age attribute in UTKface. 
For SoFCLR and SimCLR-CCL, we still assume 5\% of training examples of CelebA have sensitive attribute information. The results  shown in Table~\ref{tab:cross-set-ups} indicate that our SoFCLR still performs the best on all fairness metrics expcet for IntraAUC, while maintaining similar classification accuracy. The standard supervised  method CE using all training examples clearly has better accuracy but is much less fair.

\begin{wrapfigure}{R}{3in}
    \centering
    \vspace{-0.35in}
    \includegraphics[width=0.44\linewidth] {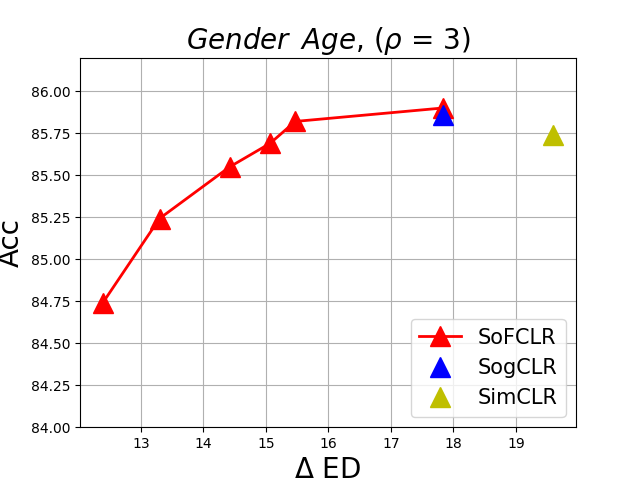}    
    \includegraphics[width=0.44\linewidth]{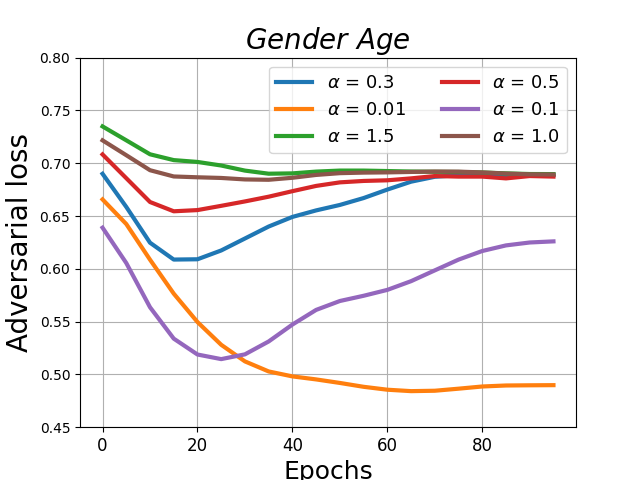}
    \vspace{-0.1in}
    \caption{ SoFCLR accuracy vs $\Delta$ ED balance  (left), and adversarial loss evolution with varying $\alpha$ (right) on UTKFace data.} \vspace*{-0.1in}
    \label{fig:robust}
\end{wrapfigure}

\vspace*{-0.15in}
\subsection{Fair representation visualization}\vspace*{-0.1in}
We compare SoFCLR with SimCLR and SogCLR on the CelebA  dataset using  `Attractive' as the target label and `Male' as the sensitive attribute.  We extract 1000 samples from the test dataset and generate a t-SNE visualization, as depicted in Figure~\ref{fig:TSNE_CelebA}. The results indicate that the learned representations by SimCLR and SogCLR are highly related to the sensitive attribute (gender) as indicated by the color. 
In contrast, the learned representations of  SoFCLR removes the disparity impact of the sensitive attribute information. Nevertheless, the learned representations still maintain discriminative power for classifying target labels (attractive vs not-attractive)  indicated by the shape. 

\vspace*{-0.2in}
\subsection{Effectiveness of fairness regularizer}\vspace*{-0.1in}
To demonstrate the impact of the adversarial fairness regularization in our method, we conduct an experiment on UTKFace data by varying the regularization parameter $\alpha$.  We use the the target label `gender' and the sensitive attribute `age'. We vary $\alpha$ across a range of values, specifically, $\{0, 0.1, 0.3, 0.5, 0.7, 0.9, 1\}$.  Notably, when $\alpha = 0$, our proposed algorithm reduces to SogCLR. We report a Pareto curve of accuracy vs $\Delta$ ED in Figure~\ref{fig:robust} (left) and evolution of the adversarial loss for predicting the sensitive attribute in Figure~\ref{fig:robust} (right). We can see that by increasing the value $\alpha$, our algorithm can effectively control the adverarial loss, which will make the downstream predictive model more fair as shown in the left figure. To ensure the comprehensiveness of our experiments, we also compare our SSL method with a VAE-based approach and extend our method to multi-valued sensitive attributes in Appendix~\ref{sec:more-expri}.
\vspace*{-0.15in}
\section{Conclusions}
\vspace*{-0.1in}
We have proposed a zero-sum game for fair self-supervised representation learning. We provided theoretical justification about the distributional representation fairness, and developed a stochastic algorithm for solving the minimax zero-sum game problems and established a convergence guarantee under mild conditions. Experiments on face image datasets demonstrate the effectiveness of our algorithm. One limitation of the work is that it focuses on single modality and multi-modality data would be an interesting future work.

\bibliography{NARL, ref}
\bibliographystyle{plain}

\nocite{langley00}


\newpage
\appendix

\section{Fairness Verification}\label{apd:fair_verification}
\begin{proof}[Proof of Theorem~\ref{thm:fair_verification}]

Denote by $D_k(E(\x))$ as the $k$-th element of $D(E(\x))$ for the $k$-th value of the sensitive attribute. Then $\sum_{k=1}^K D_k(E(\x))=1$. We define the minimax problem as: \begin{align*} \min_{E}\max_{D}\mathbb E_{\x,a}& \left[\sum_{k=1}^K\delta(a, k)\log D_k(E(\x)) \right] \end{align*} Let us first fix $E$ and optimize $D$. The objective is equivalent to 
\begin{align*} \mathbb E_{x}\sum_{k=1}^Kp(a=k|E(\x))\log D_k(E(\x)) \end{align*} By maximizing $D(E(\x))$, we have $D_k(E(\x)) = p(a=k|E(\x))$. Then we have the following objective for $E$: 
\begin{align*} \mathbb E_{\x,a}& \left[\sum_{k=1}^K\delta(a, k)\log p(a=k|E) \right] \ = \mathbb E_{\x, a} \left[\sum_{k=1}^K\delta(a, k)\log \frac{p(E|a=k)p(a=k)}{p(E)} \right]\ \\
& = \mathbb E_{\x, a}\left[\sum_{k=1}^K\delta(a, k)\log p(a=k)\right] + \mathbb E_{\x, a}\left[\sum_{k=1}^K\delta(a, k)\log \frac{p(E|a=k)}{p(E)}\right]\ \\
&= C + \mathbb E_{\x, a}\left[\log \frac{p(E|a)}{p(E)} \right] = C + \mathbb E_{a}\mathbb E_{p(E|a)}\left[\log \frac{p(E|a)}{p(E)} \right] = C +\mathbb E_{a}[\text{KL}(p(E|a), p(E))] \end{align*} where $C$ is independent of $E$. Hence by minimizing over $E$ we have the optimal $E_*$ satisfying $p(E_*|a) = p(E_*)$. As a result, $[D_*(E_*(\x))]_k = p(a=k|E_*(\x)) = p(a=k)$.
\end{proof}

\section{Convergence Analysis of SoFCLR}\label{apd:analysis}
For simplicity, we use the following notations in this section,
\begin{equation*}
    \begin{aligned}
        &F_1(\w) = \E_{\x,\A,\A'} f_1(\w;\x,\A,\A'),\\
        &\phi(\w,\w') = \E_{(\x,a)\sim\D_a,\A\sim\P}\{\phi(\w,\w';\A(\x),a)\},\\
        &g(\w) = [g(\w;\x_1,\S_1^-),\dots, g(\w;\x_n,\S_n^-)],\\
        &F_2(\w) = \frac{1}{n}\sum_{i=1}^n f_2(g(\w;\x_i,\S_i^-)),\\
        &F_3(\w) = \max_{\w'\in \R^{d'}}\alpha \phi(\w,\w').
    \end{aligned}
\end{equation*}
Then the problem
\begin{equation}\label{prob:a1}
\begin{aligned}
&\min_{\w\in \R^{d}} \max_{\w'\in \R^{d'}}F(\w,\w')  := F_{\text{GCL}}(\w) +\alpha F_{\fair}(\w,\w')\\
&=\E_{\x,\A,\A'} f_1(\w;\x,\A,\A')+ \frac{1}{n}\sum_{i=1}^n f_2(g(\w;\x_i,\S_i^-)) + \alpha \E_{(\x,a)\sim\D_a,\A\sim\P}\{\phi(\w,\w';\A(\x),a)\}\\
 \end{aligned}
\end{equation}
can be written as
\begin{equation}
\begin{aligned}
\min_{\w\in \R^{d}} \max_{\w'\in \R^{d'}}F(\w,\w') =\min_{\w\in \R^{d}} \Phi(\w)
 \end{aligned}
\end{equation}
where
\begin{equation*}
    \Phi(\w) = F_1(\w) +F_2(\w) +F_3(\w).
\end{equation*}

\begin{ass}\label{ass:a1}
    We make the following assumptions.
    \begin{enumerate}[label=(\alph*)]
        \item $\Phi(\cdot)$ is $L$-smooth.
        \item $f_2(\cdot)$ is differentiable, $L_{f_2}$-smooth, $C_{f_2}$-Lipchitz continuous;
        \item $g(\cdot;\x_i,\S_i^-)$ is differentiable, $C_g$-Lipchitz continuous for all $\x_i\in \D$.
        \item $\nabla \phi(\w,\w')$ is $L_\phi$-Lipschitz continuous.
        \item There exists $\Delta_\Phi<\infty$ such that $\Phi(\w_1)-\Phi^*\leq \Delta_\Phi$.
        \item The stochastic estimators $g(\w; \A(\x_i), \x_i^-)$, $f_1(\w;\x,\A,\A')$, $\nabla_\w g(\w; \A(\x_i), \x_i^-)$, $\nabla_\w\phi(\w,\w';\A(\x_i),a_i)$, $\nabla_{\w'}\phi(\w,\w';\A(\x_i),a_i)$ have bounded variance $\sigma^2$.
        \item For any $\w$, $-\phi(\w,\cdot)$ is $\lambda$-one-point strongly convex, i.e., 
        \begin{equation*}
        -(\w' - \w'_* )^{\top}\nabla_{\w'}\phi(\w,\w')\geq \lambda \|\w' - \w'_*\|^2,
        \end{equation*} 
        where $\w'_* \in \arg\min_{\v} \phi(\w,\v)$ is one optimal solution closest to $\w'$.
    \end{enumerate}
\end{ass}

Note that Assumption~\ref{ass:a1} is generalized from Assumption~\ref{ass:1}, which fits for all problems in the formulation of Problem~(\ref{prob:a1}). Next, we show that Assumption~\ref{ass:1} implies Assumption~\ref{ass:a1}. In fact, it suffices to show that $(b,c)$ of Assumption~\ref{ass:1} implies $(b,c,d)$ of Assumption~\ref{ass:a1}. The exact formulation of $f_2(\cdot) = \tau \log(\epsilon_0'+\cdot)$ naturally leads to its differentiability, Lipschitz continuity and smoothness. Given the formulation $g(\w, \x_i, \S_i^{-}) = \E_{\tx\sim\S_i^-}\E_{\A}\exp(E(\A(\x_i))^{\top}E_1(\tx)/\tau)$, the boundedness and Lipschitz continuity of $E(\x), E_1(\x)$, the gradient
\begin{equation*}
    \nabla_\w g(\w, \x_i, \S_i^{-}) = \E_{\tx\sim\S_i^-}\E_{\A}\frac{\nabla E(\A(\x_i))^{\top}E_1(\tx)+E(\A(\x_i))^{\top} \nabla E_1(\tx)}{\tau}\exp(E(\A(\x_i))^{\top}E_1(\tx)/\tau)
\end{equation*}
is bounded by a finite constant. 
The smoothness of $\phi(\w,\w')$ follows from the smoothness and Lipschitz continuity of $\D_{\w'}$ and $E_{\w}$.

It has been shown in the literature that smoothness and one-point strong convexity imply PL condition \cite{DBLP:conf/nips/Yuan0JY19}.
\begin{lem}\label{lem:4}[Lemma 9 in \cite{DBLP:conf/nips/Yuan0JY19}]
    Suppose $h$ is $L_h$-smooth and $\mu$-one-point strongly convex w.r.t. $\w^*$ with $\nabla h(\w^*)=0$, then
    \begin{equation*}
        \|\nabla h (\w)\|^2\geq \frac{2\mu}{L_h}(h(\w)-h(\w^*)).
    \end{equation*}
\end{lem}
Thus, under Assumption~\ref{ass:a1}, $-\phi(\w,\cdot)$ satisfies $\frac{\lambda}{L_\phi}$-PL condition. Following a similar proof to Lemma 17 in \cite{guo2021novel}, we have the following lemma.

\begin{lem}\label{lem:1}
Suppose Assumption~\ref{ass:a1} holds. Consider the update $\w'_{t+1}=\w'_t+\eta'\v_{t+1}$ in Algorithm~\ref{alg:NCSF}. With $\eta'\leq \min\left\{\lambda/(2L_\phi^2),4/\lambda\right\}$, we have that for any $\w'(\w_t)\in \argmax_{\w'}\phi(\w_t,\w')$ there is a $\w'(\w_{t+1})\in \argmax_{\w'}\phi(\w_{t+1},\w')$ such that
\begin{equation*}
    \E[\|\w'_{t+1}-\w'(\w_{t+1})\|^2]\leq (1-\frac{\eta'\lambda}{4})\E[\|\w'_t-\w'(\w_t)\|^2]+\frac{\eta'^2\sigma^2}{|\B_a|}+\frac{2L_\phi^4}{\eta'\lambda^3}\E[\|\w_t-\w_{t+1}\|^2]
\end{equation*}
\end{lem}

\begin{lem}[Lemma 9 in \cite{qiu2023largescale}]\label{lem:2}
Suppose Assumption~\ref{ass:a1} holds. Consider the update
\begin{equation*}
    \begin{aligned}
        &\u_{i, t+1} =\begin{cases} (1-\gamma)\u_{i, t}+ \frac{\gamma}{2}\left[ g(\w_t;\A(\x_i), \B^-_i)+g(\w_t;\A'(\x_i), \B^-_i)\right],\quad &\x_i\in \B\\\u_{i, t}, &\x_i\not\in \B
        \end{cases}
    \end{aligned}
\end{equation*}
With $\gamma <1/2$, we have
    \begin{equation*}
        \E[\frac{1}{n}\|\u^{t+1}-g(\w_{t+1})\|^2]\leq (1-\frac{|\B|\gamma}{2n})\E[\frac{1}{n}\|\u^t-g(\w_t)\|^2]+\frac{4B\gamma^2\sigma^2}{2n|\B_i^-|} +\frac{4nC_g^2}{|\B|\gamma}\E[\|\w_t-\w_{t+1}\|^2].
    \end{equation*}
\end{lem}

\begin{lem}\label{lem:3}
Suppose Assumption~\ref{ass:a1} holds. Considering the update $\w_{t+1} = \w_t -\eta\widetilde{\m}_{t+1}$, with $\eta\leq 1/(2L)$, we have
    \begin{equation*}
    \begin{aligned}
        \Phi(\w_{t+1})
        &\leq\Phi(\w_t) + \frac{\eta}{2} \|\nabla\Phi(\w_t)-\tilde{\m}_{t+1}\|^2-\frac{\eta}{2}\|\nabla \Phi(\w_t)\|^2-\frac{\eta}{4}\|\tilde{\m}_{t+1}\|^2.
    \end{aligned}
\end{equation*}
\end{lem}

Now we present a formal statement of Theorem~\ref{thm:1}.
\begin{thm}\label{thm:a1}
Suppose Assumption~\ref{ass:a1} holds. With  $\beta=\min\left\{1,\frac{\min\{|\B|,|\B_i^-|,|\B_a|\}\epsilon^2}{12 C_1}\right\}$, $\gamma= \min\left\{\frac{1}{2},\frac{|\B_i^-|\epsilon^2}{384 C_g^2L_{f_2}^2\sigma^2  } \right\}$, $\eta'= \min\left\{\frac{\lambda}{2L_\phi^2},\frac{ \lambda|\B_a|\epsilon^2}{384\alpha^2 L_\phi^2 \sigma^2}\right\}$, $\eta= \min\left\{\frac{1}{2L},\frac{2\beta}{L},\frac{|\B|\gamma}{32L_{f_2}C_g^2 n}, \frac{\eta'\lambda}{64\alpha L_\phi \kappa}\right\}$, after
\begin{equation*}
\begin{aligned}
    T\geq \frac{4\Lambda_P}{\eta \epsilon^2}
    &= \frac{4\Lambda_P}{ \epsilon^2}\max\left\{2L,\frac{L}{2\beta},\frac{32L_{f_2}C_g^2 n}{|\B|\gamma}, \frac{64\alpha L_\phi \kappa}{\eta'\lambda}\right\}\\
    &= \frac{4\Lambda_P}{ \epsilon^2}\max\left\{2L,\frac{6 C_1L}{\min\{|\B|,|\B_i^-|,|\B_a|\}\epsilon^2},\frac{64L_{f_2}C_g^2 n}{|\B|},\frac{12288 n C_g^4L_{f_2}^3\sigma^2  }{|\B||\B_i^-|\epsilon^2}, 
    \frac{128\alpha L_\phi^3 \kappa}{\lambda^2},\frac{24576\alpha^3 L_\phi^3\kappa \sigma^2}{ \lambda^2|\B_a|\epsilon^2}\right\}\\
\end{aligned}
\end{equation*}
iterations, Algorithm~\ref{alg:NCSF} ensures $\frac{1}{T}\sum_{t=1}^T \E[\|\nabla \Phi(\w_t)\|^2]\leq \epsilon^2$.
\end{thm}

\subsection{Proof of Lemma~\ref{lem:1}}
To prove Lemma~\ref{lem:1}, we need the following lemma.
\begin{lem}\label{lem:5}[Lemma A.3 in \cite{nouiehed2019solving}]
    Assume $h(x,y)$ is $L_h$ smooth and $-h(x,y)$ satisfies $\mu$-PL condition w.r.t. $y$. For any $x_1,x_2$ and $y(x_1)\in \argmax_{y'}h(x_1,y')$, there exists $y(x_2)\in \argmax_{y'} h(x_2,y')$ such that
    \begin{equation*}
        \|y(x_1)-y(x_2)\|\leq \frac{L_h}{2\mu}\|x_1-x_2\|.
    \end{equation*}
\end{lem}

Recall that we assume $-\phi(\w,\w')$ to be $\lambda$-one-point strongly convex in $\w'$, and Lemma~\ref{lem:4} implies that $-\phi(\w,\w')$ satisfies $\frac{\lambda}{L_\phi}$-condition w.r.t. $\w'$. Thus, by Lemma~\ref{lem:5}, $\w'(\w)\in \argmax_{\w'}\phi(\w,\w')$ is $\frac{L_\phi^2}{2\lambda}$-Lipschitz continuous. Then we follow the proof of Lemma 17 in \cite{guo2021novel} to prove Lemma~\ref{lem:1}. We would like to emphasize that concavity is irrelevant in Lemma~\ref{lem:1}, while it is required in Lemma 17 in \cite{guo2021novel}.

\begin{proof}
    \begin{equation}\label{ineq:a5}
        \begin{aligned}
            &\E[\|\w'_{t+1}-\w'(\w_t)\|^2]\\
            & = \E[\|\w'_t+\eta'\v_{t+1}-\w'(\w_t)\|^2]\\
            & = \E[\|\w'_t+\eta'\v_{t+1}-\w'(\w_t)+\eta'\nabla_{\w'}\phi(\w_t,\w'_t)-\eta'\nabla_{\w'}\phi(\w_t,\w'_t)\|^2]\\
            & =  \E[\|\w'_t-\w'(\w_t)+\eta'\nabla_{\w'}\phi(\w_t,\w'_t)\|^2] +\eta'^2\E[\|\v_{t+1}-\eta'\nabla_{\w'}\phi(\w_t,\w'_t)\|^2]\\
            &\leq \E[\|\w'_t-\w'(\w_t) +\eta'\nabla_{\w'}\phi(\w_t,\w'_t)\|^2] + \frac{\eta'^2\sigma^2}{2|\B_a|}
        \end{aligned}
    \end{equation}
    where
    \begin{equation*}
        \begin{aligned}
            & \E[\|\w'_t-\w'(\w_t) +\eta'\nabla_{\w'}\phi(\w_t,\w'_t)\|^2]\\
            & = \E[\|\w'_t-\w'(\w_t) - \eta'\nabla_{\w'}\phi(\w_t,\w'(\w_t))+\eta'\nabla_{\w'}\phi(\w_t,\w'_t)\|^2]\\
            & = \E[\|\w'_t-\w'(\w_t)\|^2] + \eta'^2\E[\|\nabla_{\w'}\phi(\w_t,\w'(\w_t))-\nabla_{\w'}\phi(\w_t,\w'_t)\|^2] \\
            &\quad + 2\eta' \langle\w'_t-\w'(\w_t), \nabla_{\w'}\phi(\w_t,\w'_t)\rangle\\
            & \stackrel{(a)}{\leq} \E[\|\w'_t-\w'(\w_t)\|^2] + \eta'^2\E[\|\nabla_{\w'}\phi(\w_t,\w'(\w_t))-\nabla_{\w'}\phi(\w_t,\w'_t)\|^2] -\eta'\lambda\E[\|\w'_t-\w'(\w_t)\|^2]\\
            & \stackrel{(b)}{\leq} (1-\frac{\eta'\lambda}{2}) \E[\|\w'_t-\w'(\w_t)\|^2] 
        \end{aligned}
    \end{equation*}
    where inequality $(a)$ uses the $\lambda$-one-point strong convexity of $-\phi(\w,\cdot)$, and inequality $(b)$ uses the assumption $\eta'\leq \lambda/(2L_\phi^2)$.

    Then we have
    \begin{equation*}
        \begin{aligned}
            &\E[\|\w'_{t+1}-\w'(\w_{t+1})\|^2]\\
            &\leq (1+\frac{\eta'\lambda}{4})\E[\|\w'_{t+1}-\w'(\w_t)\|^2] + (1+\frac{4}{\eta'\lambda})\E[\|\w'(\w_{t+1})-\w'(\w_t)\|^2]\\
            &\stackrel{(c)}{\leq} (1-\frac{\eta'\lambda}{4})\E[\|\w'_t-\w'(\w_t)\|^2]  + (1+\frac{\eta'\lambda}{4}) \frac{\eta'^2\sigma^2}{2|\B_a|} + (1+\frac{4}{\eta'\lambda}) \frac{L_\phi^4}{4\lambda^2}\E[\|\w'_{t+1}-\w'_t\|^2]\\
            &\stackrel{(d)}{\leq} (1-\frac{\eta'\lambda}{4})\E[\|\w'_t-\w'(\w_t)\|^2]  +  \frac{\eta'^2\sigma^2}{|\B_a|} +\frac{2L_\phi^4}{\eta'\lambda^3}\E[\|\w'_{t+1}-\w'_t\|^2].
        \end{aligned}
    \end{equation*}
where inequality $(c)$ uses the $\frac{L_\phi^2}{2\lambda}$-Lipschitz continuity of $\w'(\cdot)$ and inequality~\ref{ineq:a5}, and inequality $(d)$ uses the assumption $\eta'\leq 4/\lambda$.
\end{proof}

\subsection{Proof of Lemma~\ref{lem:3}}
\begin{proof}
By the smoothness of $\Phi(\w)$, we have
\begin{equation*}
    \begin{aligned}
        \Phi(\w_{t+1})
        &\leq \Phi(\w_t) +\langle\nabla \Phi(\w_t), \w_{t+1}-\w_t\rangle +\frac{L}{2}\|\w_{t+1}-\w_t\|^2\\
        &= \Phi(\w_t) +\eta\langle\nabla \Phi(\w_t), \tilde{\m}_{t+1}\rangle +\frac{L\eta^2}{2}\|\tilde{\m}_{t+1}\|^2\\
        &=\Phi(\w_t) + \frac{\eta}{2} \|\nabla\Phi(\w_t)-\tilde{\m}_{t+1}\|^2-\frac{\eta}{2}\|\nabla \Phi(\w_t)\|^2+(\frac{L\eta^2}{2}-\frac{\eta}{2})\|\tilde{\m}_{t+1}\|^2\\
        &\leq\Phi(\w_t) + \frac{\eta}{2} \|\nabla\Phi(\w_t)-\tilde{\m}_{t+1}\|^2-\frac{\eta}{2}\|\nabla \Phi(\w_t)\|^2-\frac{\eta}{4}\|\tilde{\m}_{t+1}\|^2
    \end{aligned}
\end{equation*}
where the last inequality uses $\eta L\leq 1/2$.
    
\end{proof}

\subsection{Proof of Theorem~\ref{thm:a1}}
\begin{proof}
The formulation of the gradient is given by
\begin{equation*}
    \begin{aligned}
        \nabla\Phi(\w_t) &= \nabla F_1(\w_t) + \nabla F_2(\w_t) + \nabla F_3(\w_t)\\
        & = \E_{\x,\A,\A'} \nabla f_1(\w_t;\x,\A,\A')+ \frac{1}{n}\sum_{i=1}^n \nabla g(\w_t;\x_i,\S_i^-)\nabla f_2(g(\w_t;\x_i,\S_i^-))
        + \alpha \nabla_{\w}\phi(\w_t,\w'(\w_t))
    \end{aligned}
\end{equation*}
Recall the formulation of $\m_{t+1}$,
\begin{equation*}
    \begin{aligned}
        \m_{t+1}& = \frac{1}{|\B|}\sum_{\x_i\in \B}\nabla f_1(\w_t;\x,\A,\A') + \frac{1}{|\B|}\sum_{\x_i\in \B} \big[\nabla_\w g(\w_t; \A(\x_i), \B_i^-) + \nabla_\w g(\w_t;\A(\x_i),\B_i^-)\big]\nabla f_2(\u_{i,t}) \\
& \quad +\frac{\alpha}{2|\B_a|}\sum_{\x_i\in \B_a}
\bigg\{\nabla_\w\phi(\w_t,\w'_t;\A(\x_i),a_i)+\nabla_\w\phi(\w_t,\w'_t;\A', \x_i,a_i)\bigg\}
    \end{aligned}
\end{equation*}

Define
\begin{equation*}
    \begin{aligned}
        G_{t+1} &= \E_{\x,\A,\A'} \nabla f_1(\w_t;\x,\A,\A')+ \frac{1}{n}\sum_{i=1}^n \nabla g(\w_t;\x_i,\S_i^-)\nabla f_2(\u_{i,t})+ \alpha \nabla_{\w}\phi(\w_t,\w'_t)
    \end{aligned}
\end{equation*}
so that we have $\E_t[\m_{t+1}] = G_{t+1}$, where $\E_t$ denotes the expectation over the randomness at $t$-th iteration.

Consider
\begin{equation*}
    \begin{aligned}
        &\E_t\left[\|\nabla\Phi(\w_t)-\tilde{\m}_{t+1}\|^2\right]\\
        & = \E_t\left[\|\nabla\Phi(\w_t)-(1-\beta)\tilde{\m}_t-\beta \m_{t+1}\|^2\right]\\
        & =\E_t\left[ \|(1-\beta)(\nabla\Phi(\w_{t-1})-\tilde{\m}_t)+(1-\beta)(\nabla\Phi(\w_t)-\nabla\Phi(\w_{t-1}))+\beta (\nabla\Phi(\w_t)-G_{t+1})+\beta(G_{t+1}-\m_{t+1})\|^2\right]\\
        & \stackrel{(a)}{=}  \|(1-\beta)(\nabla\Phi(\w_{t-1})-\tilde{\m}_t)+(1-\beta)(\nabla\Phi(\w_t)-\nabla\Phi(\w_{t-1}))+\beta (\nabla\Phi(\w_t)-G_{t+1})\|^2\\
        &\quad +\beta^2\E_t\left[\|G_{t+1}-\m_{t+1}\|^2\right]\\
        &\stackrel{(b)}{\leq}(1+\beta)(1-\beta)^2\|\nabla\Phi(\w_{t-1})-\tilde{\m}_t\|^2+2(1+\frac{1}{\beta})\big[(1-\beta)^2\|\nabla\Phi(\w_t)-\nabla\Phi(\w_{t-1})\|^2+\beta^2 \|\nabla\Phi(\w_t)-G_{t+1}\|^2\big]\\
        &\quad +\beta^2\E_t\left[\|G_{t+1}-\m_{t+1}\|^2\right]\\
        &\stackrel{(c)}{\leq}(1-\beta)\|\nabla\Phi(\w_{t-1})-\tilde{\m}_t\|^2+\frac{4L^2}{\beta}\|\w_t-\w_{t-1}\|^2+4\beta \|\nabla\Phi(\w_t)-G_{t+1}\|^2+\beta^2\E_t\left[\|G_{t+1}-\m_{t+1}\|^2\right]\\
    \end{aligned}
\end{equation*}
where equality $(a)$ uses $\E_t[\m_{t+1}] = G_{t+1}$, inequality $(b)$ is due to $\|a+b\|^2\leq (1+\beta)\|a\|^2+(1+\frac{1}{\beta})\|b\|^2$, inequality $(c)$ uses the assumption $\beta\leq 1$ and the smoothness of $\Phi(\cdot)$.

Furthermore, one may bound the last two terms as following.
\begin{equation*}
    \begin{aligned}
        \|\nabla\Phi(\w_t)-G_{t+1}\|^2
        &=\bigg\| \frac{1}{n}\sum_{i=1}^n \nabla g(\w_t;\x_i,\S_i^-)\nabla f_2(g(\w_t;\x_i,\S_i^-))-\frac{1}{n}\sum_{i=1}^n \nabla g(\w_t;\x_i,\S_i^-)\nabla f_2(\u_{i,t})\\
        &\quad + \alpha \big[ \nabla_{\w}\phi(\w_t,\w'(\w_t))- \nabla_{\w}\phi(\w_t,\w'_t)\big]\bigg\|^2\\
        &\leq 2C_g^2L_{f_2}^2\frac{1}{n}\sum_{i=1}^n\|g(\w_t;\x_i,\S_i^-)-\u_{i,t}\|^2 + 2\alpha^2L_{\phi}^2\|\w'(\w_t)-\w'_t\|^2\\
        &\leq \frac{2C_g^2L_{f_2}^2}{n}\|g(\w_t)-\u_{t}\|^2 + 2\alpha^2L_{\phi}^2\|\w'(\w_t)-\w'_t\|^2.
    \end{aligned}
\end{equation*}

\begin{equation*}
    \begin{aligned}
        &\E_t\left[\|G_{t+1}-\m_{t+1}\|^2\right]\\
        & = \E_t\bigg[\bigg\|\E_{\x,\A,\A'} \nabla f_1(\w_t;\x,\A,\A')+ \frac{1}{n}\sum_{i=1}^n \nabla g(\w_t;\x_i,\S_i^-)\nabla f_2(\u_{i,t})+ \alpha \nabla_{\w}\phi(\w_t,\w'_t)\\
        &\quad -\frac{1}{|\B|}\sum_{\x_i\in \B}\nabla f_1(\w;\x,\A,\A') - \frac{1}{|\B|}\sum_{\x_i\in \B} \big[\nabla_\w g(\w_t; \A(\x_i), \B_i^-) + \nabla_\w g(\w_t;\A(\x_i),\B_i^-)\big]\nabla f_2(\u_{i,t}) \\
        & \quad -\frac{\alpha}{2|\B_a|}\sum_{\x_i\in \B_a}
        \bigg\{\nabla_\w\phi(\w_t,\w'_t;\A(\x_i),a_i)+\nabla_\w\phi(\w_t,\w'_t;\A', \x_i,a_i)\bigg\}\bigg\|^2]\\
        &\leq \frac{3\sigma^2}{|\B|}+3\E_t\bigg[\bigg\|\frac{1}{n}\sum_{i=1}^n \nabla g(\w_t;\x_i,\S_i^-)\nabla f_2(\u_{i,t})- \frac{1}{|\B|}\sum_{\x_i\in \B} \big[\nabla_\w g(\w_t; \A(\x_i), \B_i^-) + \nabla_\w g(\w_t;\A(\x_i),\B_i^-)\big]\nabla f_2(\u_{i,t}) \bigg\|^2\bigg]\\
        &\quad +3\E_t\bigg[\bigg\|\alpha \nabla_{\w}\phi(\w_t,\w'_t)-\frac{\alpha}{2|\B_a|}\sum_{\x_i\in \B_a}
        \bigg\{\nabla_\w\phi(\w_t,\w'_t;\A(\x_i),a_i)+\nabla_\w\phi(\w_t,\w'_t;\A', \x_i,a_i)\bigg\}\bigg\|^2\\
        &\leq \frac{3\sigma^2}{|\B|} + \frac{24C_g^2C_{f_2}^2}{|\B|}+\frac{6C_{f_2}^2\sigma^2}{2|\B_i^-|}+\frac{3\alpha^2\sigma^2}{2|\B_a|}\leq \frac{C_1}{\min\{|\B|,|\B_i^-|,|\B_a|\}}
    \end{aligned}
\end{equation*}
where $C_1 = \max\{9\sigma^2+72C_g^2C_{f_2}^2,9C_{f_2}^2\sigma^2,\frac{9\alpha^2\sigma^2}{2} \}$.

For simplicity, we denote
\begin{equation*}
    \begin{aligned}
        &\delta_\m^t = \|\nabla\Phi(\w_{t})-\tilde{\m}_{t+1}\|^2,\quad \delta_{\w'}^t = \|\w'_t-\w'(\w_t)\|^2,\quad \delta_{\u}^t = \frac{1}{n}\|\u^t-g(\w_t)\|^2.
    \end{aligned}
\end{equation*}

Combining the results above yields
\begin{equation}\label{ineq:a1}
    \begin{aligned}
        &\E\left[\delta_\m^{t+1}\right]\leq(1-\beta)\E\left[\delta_\m^t\right]+\frac{4L^2\eta^2}{\beta}\E[\|\tilde{\m}_{t+1}\|^2]+4\beta \left(2C_g^2L_{f_2}^2 \E[\delta_{\u}^{t+1}] + 2\alpha^2L_{\phi}^2\E[\delta_{\w'}^{t+1}]\right)+\frac{\beta^2C_1}{\min\{|\B|,|\B_i^-|,|\B_a|\}}\\
    \end{aligned}
\end{equation}

By Lemma~\ref{lem:1} we have
\begin{equation}\label{ineq:a2}
    \E[\delta_{\w'}^{t+1}]\leq (1-\frac{\eta'\lambda}{4})\E[\delta_{\w'}^{t}]+\frac{\eta'^2\sigma^2}{|\B_a|}+\frac{2L_\phi^4\eta^2}{\eta'\lambda^3}\E[\|\tilde{\m}_{t+1}\|^2]
\end{equation}


By Lemma~\ref{lem:2} we have
\begin{equation}\label{ineq:a3}
        \E\left[\delta_{\u}^{t+1}\right]\leq (1-\frac{|\B|\gamma}{2n})\E\left[\delta_{\u}^t\right]+\frac{4|\B|\gamma^2\sigma^2}{2n|\B_i^-|} +\frac{4nC_g^2\eta^2}{|\B|\gamma}\E[\|\tilde{\m}_{t+1}\|^2].
\end{equation}

By Lemma~\ref{lem:3} we have
\begin{equation}\label{ineq:a4}
    \begin{aligned}
        \E[\Phi(\w_{t+1})]
        &\leq\E[\Phi(\w_t)] + \frac{\eta}{2} \E[\delta_\m^{t}]-\frac{\eta}{2}\E[\|\nabla \Phi(\w_t)\|^2]-\frac{\eta}{4}\E[\|\tilde{\m}_{t+1}\|^2]
    \end{aligned}
\end{equation}
Summing $(\ref{ineq:a4})$, $\frac{\eta}{\beta}\times (\ref{ineq:a1})$, $\frac{16 C_g^2L_{f_2}^2n\eta}{|\B|\gamma}\times (\ref{ineq:a3})$, $\frac{32\alpha^2 L_\phi^2 \eta}{\eta' \lambda} \times (\ref{ineq:a2})$ yields




\begin{equation}
    \begin{aligned}
        &\E[\Phi(\w_{t+1})]-\Phi^*+\frac{\eta}{\beta}\E\left[\delta_\m^{t+1}\right]+\frac{16 C_g^2L_{f_2}^2n\eta}{|\B|\gamma}(1-\frac{|\B|\gamma}{2n})\E\left[\delta_{\u}^{t+1}\right]+\frac{32\alpha^2 L_\phi^2 \eta}{\eta' \lambda}(1-\frac{\eta'\lambda}{4})\E[\delta_{\w'}^{t+1}]\\
        &\leq\E[\Phi(\w_t)]-\Phi^*+\frac{\eta}{\beta}(1-\frac{\beta}{2})\E\left[\delta_\m^t\right]+\frac{16 C_g^2L_{f_2}^2n\eta}{|\B|\gamma}(1-\frac{|\B|\gamma}{2n})\E\left[\delta_{\u}^t\right]+\frac{32\alpha^2 L_\phi^2 \eta}{\eta' \lambda}(1-\frac{\eta'\lambda}{4})\E[\delta_{\w'}^{t}]\\
        &\quad -\frac{\eta}{2}\E[\|\nabla \Phi(\w_t)\|^2]+\eta\left(\frac{4L^2\eta^2}{\beta^2}+\frac{64L_{f_2}^2C_g^4 n^2\eta^2}{|\B|^2\gamma^2}+\frac{64\alpha^2 L_\phi^6\eta^2}{\eta'^2\lambda^4}-\frac{1}{4}\right)\E[\|\tilde{\m}_{t+1}\|^2]\\
        &\quad +\frac{\eta\beta C_1}{\min\{|\B|,|\B_i^-|,|\B_a|\}} +\frac{32 C_g^2L_{f_2}^2\gamma\sigma^2 \eta }{|\B_i^-|}  +\frac{32\alpha^2 L_\phi^2 \eta\eta'\sigma^2}{ \lambda|\B_a|}
    \end{aligned}
\end{equation}

Set
\begin{equation*}
    \eta= \min\left\{\frac{1}{2L},\frac{2\beta}{L},\frac{|\B|\gamma}{32L_{f_2}C_g^2 n}, \frac{\eta'\lambda^2}{16\alpha L_\phi^3}\right\}
\end{equation*}
so that $\frac{4L^2\eta^2}{\beta^2}+\frac{64L_{f_2}^2C_g^4 n^2\eta^2}{|\B|^2\gamma^2}+\frac{64\alpha^2 L_\phi^6\eta^2}{\eta'^2\lambda^4}-\frac{1}{4}\leq 0$. Define potential function
\begin{equation*}
    P_t = \E[\Phi(\w_t)]-\Phi^*+\frac{\eta}{\beta}\E\left[\delta_\m^t\right]+\frac{16 C_g^2L_{f_2}^2n\eta}{|\B|\gamma}(1-\frac{|\B|\gamma}{2n})\E\left[\delta_{\u}^t\right]+\frac{32\alpha^2 L_\phi^2 \eta}{\eta' \lambda}(1-\frac{\eta'\lambda}{4})\E[\delta_{\w'}^{t}],
\end{equation*}
then we have

\begin{equation}
    \frac{1}{T}\sum_{t=1}^T \E[\|\nabla \Phi(\w_t)\|^2] \leq \frac{2P_1}{\eta T} +\frac{2\beta C_1}{\min\{|\B|,|\B_i^-|,|\B_a|\}} +\frac{64 C_g^2L_{f_2}^2\gamma\sigma^2  }{|\B_i^-|}  +\frac{64\alpha^2 L_\phi^2 \eta'\sigma^2}{ \lambda|\B_a|}
\end{equation}

If we initialize  $\u_{i, 1}= \frac{1}{2}\left[ g(\w_1;\A(\x_i), \B^-_i)+g(\w_1;\A'(\x_i), \B^-_i)\right]$ for all $i$, then we have $\E[\delta_{\u}^1]\leq \frac{\sigma^2}{2|\B^-_i|}$. Moreover, we run multiple steps of stochastic gradient ascent to approximate $\w'(\w_1)$ so that $\E[\delta_{\w'}^1]\leq 1$. We set $\tilde{\m}_1 = \m_2$ so that 
\begin{equation*}
    \begin{aligned}
        \E[\delta_{\m}^1] &= \E[\|\Phi(\w_1)-\tilde{\m}_2\|^2]\\
        &=\E[\|\Phi(\w_1)-\m_2\|^2]\\
        & = \E\bigg[\bigg\|\E_{\x,\A,\A'} \nabla f_1(\w_1;\x,\A,\A')+ \frac{1}{n}\sum_{i=1}^n \nabla g(\w_1;\x_i,\S_i^-)\nabla f_2(g(\w_1;\x_i,\S_i^-))
        + \alpha \nabla_{\w}\phi(\w_1,\w'(\w_1))\\
        &\quad -\frac{1}{|\B|}\sum_{\x_i\in \B}\nabla f_1(\w_1;\x,\A,\A') - \frac{1}{|\B|}\sum_{\x_i\in \B} \big[\nabla_\w g(\w_1; \A(\x_i), \B_i^-) + \nabla_\w g(\w_1;\A(\x_i),\B_i^-)\big]\nabla f_2(\u_{i,t}) \\
        & \quad -\frac{\alpha}{2|\B_a|}\sum_{\x_i\in \B_a}
        \bigg\{\nabla_\w\phi(\w_1,\w'_1;\A(\x_i),a_i)+\nabla_\w\phi(\w_1,\w'_1;\A', \x_i,a_i)\bigg\}\bigg\|^2\bigg]\\
        &\leq \frac{3\sigma^2}{|\B|}+9C_g^2L_{f_2}^2\E[\delta_{\u}^1]+\frac{9C_g^2C_{f_2}^2}{|\B|}+\frac{\sigma^2C_{f_2}^2}{2|\B_i^-|}+6\alpha^2L_\phi^2 \E[\delta_{\w'}^1] +\frac{6\alpha^2\sigma^2}{2|\B_a|}\\
        &\leq \frac{3\sigma^2}{|\B|}+9C_g^2L_{f_2}^2\frac{\sigma^2}{2|\B^-_i|}+9\frac{C_g^2C_{f_2}^2}{|\B|}+\frac{\sigma^2C_{f_2}^2}{2|\B_i^-|}+6\alpha^2L_\phi^2  +\frac{6\alpha^2\sigma^2}{2|\B_a|}=:C_2
    \end{aligned}
\end{equation*}

Define $\Lambda_P :=\Delta_\Phi+\frac{2C_2}{L} +\frac{L_{f_2}\sigma^2}{4|\B^-_i|}+\frac{\alpha \lambda}{2} \geq P_1$. Then
\begin{equation}
    \frac{1}{T}\sum_{t=1}^T \E[\|\nabla \Phi(\w_t)\|^2] \leq \frac{2\Lambda_P }{\eta T} +\frac{2\beta C_1}{\min\{|\B|,|\B_i^-|,|\B_a|\}} +\frac{64 C_g^2L_{f_2}^2\gamma\sigma^2  }{|\B_i^-|}  +\frac{64\alpha^2 L_\phi^2 \eta'\sigma^2}{ \lambda|\B_a|}.
\end{equation}

Set
\begin{equation*}
    \beta=\min\left\{1,\frac{\min\{|\B|,|\B_i^-|,|\B_a|\}\epsilon^2}{12 C_1}\right\} ,\quad \gamma= \min\left\{\frac{1}{2},\frac{|\B_i^-|\epsilon^2}{384 C_g^2L_{f_2}^2\sigma^2  } \right\} ,\quad \eta'= \min\left\{\frac{4}{\lambda}, \frac{\lambda}{2L_\phi^2},\frac{ \lambda|\B_a|\epsilon^2}{384\alpha^2 L_\phi^2 \sigma^2}\right\},
\end{equation*}
then after
\begin{equation*}
\begin{aligned}
    T\geq \frac{4\Lambda_P}{\eta \epsilon^2}
    &= \frac{4\Lambda_P}{ \epsilon^2}\max\left\{2L,\frac{L}{2\beta},\frac{32L_{f_2}C_g^2 n}{|\B|\gamma}, \frac{16\alpha L_\phi^3}{\eta'\lambda^2}\right\}\\
    &= \frac{4\Lambda_P}{ \epsilon^2}\max\left\{2L,\frac{6 C_1L}{\min\{|\B|,|\B_i^-|,|\B_a|\}\epsilon^2},\frac{64L_{f_2}C_g^2 n}{|\B|},\frac{12288 n C_g^4L_{f_2}^3\sigma^2  }{|\B||\B_i^-|\epsilon^2},\frac{4\alpha L_{\phi}^3}{\lambda} , \frac{32\alpha L_\phi^5}{\lambda^3},\frac{6144\alpha^3 L_\phi^5 \sigma^2}{ \lambda^3|\B_a|\epsilon^2}\right\}\\
\end{aligned}
\end{equation*}
iterations, we have $\frac{1}{T}\sum_{t=1}^T \E[\|\nabla \Phi(\w_t)\|^2]\leq \epsilon^2$.
    
\end{proof}




\section{More Details about Experiments}\label{sec:expm}

In all our experiments, we tuned the learning rates ($\text{lr}$) for the Adam optimizer~\cite{kingma2014adam} within the range $\{1e$-$3, 1e$-$4, 1e$-$5\}$. The batch size is set to 128 for CelebA and 64 for UTKFace. For baseline end-to-end regularized methods that require task labels (CE, CE+EOD, CE+DPR, CE+EQL, and CE+EQL), we trained for 100 epochs, and the regularizer weights were tuned in the range $\{0.1, 0.3, 0.5, 0.7, 0.9, 1\}$. For adversarial baselines with task labels (ML-AFL and Max-Ent), we trained for 100 epochs for both feature representation learning and the adversarial head optimization, sequentially. For contrastive-based baselines and our method SoFCLR, we trained for 100 epochs for contrastive learning and 20 epochs for linear evaluation, incorporating a stagewise learning rate decay strategy at the 10th epoch by a factor of 10. The combination weights in ML-AFL, Max-Ent, SimCLR+CCL, and SoFCLR were fine-tuned within the range $\{0.1, 0.3, 0.5, 0.7, 0.9, 1\}$. All results are reported based on 4 independent runs.

\section{More Experimental Results}\label{sec:more-expri}

\subsection{More quantitative performance on CelebA datasets in Section~\ref{sec:experiemnt-prediction}.}

\begin{table*}[h!]
  \caption{Results on CelebA: accuracy of predicting Big Nose and fairness metrics for two sensitive attributes, Male and Young.}
    \label{tab:celeba_ncsf_big_nose}   \centering
    \resizebox{0.90\textwidth}{!}{
    \begin{tabular}{c|c|c|c|c|c|c|c|c|c} \toprule
(Big Nose, Male)  & Acc                                        & $\Delta$ ED        & $\Delta$ EO        & $\Delta$DP                                & IntraAUC            & InterAUC            & GAUC                & WD                  & KL                  \\ \toprule
CE                & 82.21 ($\pm$ 0.42)                         & 22.37 ($\pm$ 0.35) & 28.31 ($\pm$ 0.36) & 23.01($\pm$ 0.39)                         & 0.0433 ($\pm$ 9e-4) & 0.3610 ($\pm$ 5e-3) & 0.2973($\pm$ 6e-3)  & 0.204($\pm$ 5e-3)   & 0.6771 ($\pm$ 6e-3) \\
CE + EOD          & 82.15 ($\pm$ 0.40)                         & 19.09 ($\pm$ 0.33) & 27.32 ($\pm$ 0.39) & 22.71($\pm$ 0.33)                         & 0.0463 ($\pm$ 8e-4) & 0.3671 ($\pm$ 7e-3) & 0.2985($\pm$ 7e-3)  & 0.1965($\pm$ 6e-3)  & 0.6843($\pm$ 8e-3)  \\
CE + DPR          & 82.29 ($\pm$ 0.38)                         & 18.92 ($\pm$ 0.37) & 24.31 ($\pm$ 0.40) & 22.62($\pm$ 0.38)                         & 0.0425 ($\pm$ 7e-4) & 0.3860 ($\pm$ 9e-3)  & 0.3046($\pm$ 6e-3)  & 0.1949($\pm$ 7e-3)  & 0.706($\pm$ 6e-3)   \\
CE + EQL          & 81.91 ($\pm$ 0.39)                         & 20.92 ($\pm$ 0.41) & 27.15 ($\pm$ 0.43) & 25.55($\pm$ 0.34)                         & 0.0403 ($\pm$ 8e-4) & 0.3571 ($\pm$ 7e-3) & 0.2971($\pm$ 7e-3)  & 0.1975($\pm$ 5e-3)  & 0.637($\pm$ 8e-3)   \\ 
ML-AFL            & 81.81 ($\pm$ 0.36)                         & 29.66 ($\pm$ 0.37) & 18.76 ($\pm$ 0.38) & 24.07($\pm$ 0.41)                         & 0.0502 ($\pm$ 1e-3) & 0.5660($\pm$ 6e-3)  & 0.3683 ($\pm$ 7e-3) & 0.2124($\pm$ 4e-3)  & 1.2163($\pm$ 7e-3)  \\
Max-Ent           & 81.71 ($\pm$ 0.43)                         & 16.16 ($\pm$ 0.34) & 23.98 ($\pm$ 0.46) & 16.44($\pm$ 0.45)                         & 0.0505 ($\pm$ 8e-4) & 0.526($\pm$ 7e-3)   & 0.3618($\pm$ 8e-3)  & 0.1977($\pm$ 7e-3)  & 1.1411($\pm$ 6e-3)  \\ \hline
SimCLR            & 82.72 ($\pm$ 0.37)                         & 18.49 ($\pm$ 0.39) & 28.71 ($\pm$ 0.35) & 21.13($\pm$ 0.38)                         & 0.0664 ($\pm$ 6e-4) & 0.4462($\pm$ 7e-3)  & 0.3488($\pm$ 6e-3)  & 0.1777($\pm$ 5e-3)  & 0.9986($\pm$ 7e-3)  \\
SogCLR            & 82.64 ($\pm$ 0.41)                         & 16.35 ($\pm$ 0.30) & 26.12 ($\pm$ 0.31) & 19.91($\pm$ 0.42)                         & 0.0636 ($\pm$ 7e-4) & 0.4584($\pm$ 8e-3)  & 0.3656($\pm$ 9e-3)  & 0.1728($\pm$ 8e-3)  & 0.9986($\pm$ 7e-3)  \\ 
Boyl              & 82.62 ($\pm$ 0.35) & 16.48 ($\pm$ 0.32) & 25.13 ($\pm$ 0.36) & 19.67($\pm$ 0.39)                         & 0.0647 ($\pm$ 6e-4) & 0.4567($\pm$ 6e-3)  & 0.3435($\pm$ 7e-3)  & 0.1745($\pm$ 5e-3)  & 0.8989($\pm$ 9e-3)  \\
SimCLR + CCL      & 82.11 ($\pm$ 0.38) & 15.34 ($\pm$ 0.35) & 24.24 ($\pm$ 0.33) & 16.56($\pm$ 0.36)                         & 0.0589 ($\pm$ 8e-4) & 0.3678($\pm$ 7e-3)  & 0.2897($\pm$ 6e-3)  & 0.1544($\pm$ 7e-3)  & 0.6691($\pm$ 6e-3)  \\ \hline
SoFCLR            & 81.83 ($\pm$ 0.33)                         & \textbf{8.61} ($\pm$ 0.28)  & \textbf{18.64} ($\pm$ 0.29) & \textbf{12.91}($\pm$ 0.30)                         & \textbf{0.0341} ($\pm$ 5e-4) & \textbf{0.1538}($\pm$ 5e-3)  & \textbf{0.2299}($\pm$ 4e-3)  & \textbf{0.0816}($\pm$ 5e-3)  & \textbf{0.5809}($\pm$ 4e-3)  \\ \bottomrule
(Big Nose, Young) & Acc                                        & $\Delta$ ED        & $\Delta$ EO        & $\Delta$DP                                & IntraAUC            & InterAUC            & GAUC                & WD                  & KL                  \\ \toprule
CE                & 81.78 ($\pm$ 0.38)                         & 20.96 ($\pm$ 0.55) & 20.01 ($\pm$ 0.48) & 23.22($\pm$ 0.43)                         & 0.0401 ($\pm$ 8e-4) & 0.2169($\pm$ 9e-3)  & 0.2046($\pm$ 6e-3)  & 0.1495 ($\pm$ 7e-3) & 0.2851($\pm$ 8e-3)  \\
CE + EOD          & 81.59 ($\pm$ 0.47)                         & 18.96 ($\pm$ 0.43) & 18.31 ($\pm$ 0.41) & 24.14($\pm$ 0.38)                         & 0.0389 ($\pm$ 7e-4) & 0.2345($\pm$ 8e-3)  & 0.2037($\pm$ 5e-3)  & 0.1476 ($\pm$ 7e-3) & 0.2651($\pm$ 7e-3)  \\
CE + DPR          & 82.11 ($\pm$ 0.51)                         & 20.32 ($\pm$ 0.37) & 18.01 ($\pm$ 0.47) & 23.12($\pm$ 0.41)                         & 0.0433 ($\pm$ 8e-4) & 0.3610($\pm$ 9e-3)  & 0.2973($\pm$ 4e-3)  & 0.1603 ($\pm$ 6e-3) & 0.2573($\pm$ 7e-3)  \\
CE + EQL          & 81.58 ($\pm$ 0.38)                         & 17.23 ($\pm$ 0.41) & 18.59 ($\pm$ 0.39) & 25.40($\pm$ 0.45)                         & 0.0369 ($\pm$ 9e-4) & 0.2185($\pm$ 7e-3)  & 0.2014($\pm$ 7e-3)  & 0.1477 ($\pm$ 8e-3) & 0.2753($\pm$ 9e-3)  \\
ML-AFL            & 81.66 ($\pm$ 0.44)                         & 22.51 ($\pm$ 0.35) & 16.2 ($\pm$ 0.41)  & 24.93($\pm$ 0.36)                         & \textbf{0.0360} ($\pm$ 1e-3) & 0.3331($\pm$ 5e-3)  & 0.2392($\pm$ 6e-3)  & 0.1449($\pm$ 9e-3)  & 0.4033($\pm$ 6e-3)  \\
Max-Ent           & 81.72 ($\pm$ 0.35)                         & 21.69 ($\pm$ 0.29) & 16.37 ($\pm$ 0.35) & 25.55($\pm$ 0.33)                         & 0.0487 ($\pm$ 6e-4) & 0.2919($\pm$ 8e-3)  & 0.2289($\pm$ 6e-3)  & 0.1529($\pm$ 9e-3)  & 0.3661($\pm$ 7e-3)  \\ \hline
SimCLR            & 82.57 ($\pm$ 0.40)                         & 12.59 ($\pm$ 0.31) & 17.41 ($\pm$ 0.37) & 16.70($\pm$ 0.34)                         & 0.0564 ($\pm$ 7e-4) & 0.2214($\pm$ 8e-3)  & 0.2208($\pm$ 8e-3)  & 0.1139($\pm$ 6e-3)  & 0.3325($\pm$ 6e-3)  \\
SogCLR            & 82.48($\pm$ 0.36)                          & 12.05 ($\pm$ 0.43) & 16.21 ($\pm$ 0.39) & 15.37($\pm$ 0.39)                         & 0.0559 ($\pm$ 5e-4) & 0.2333($\pm$ 6e-3)  & 0.2268($\pm$ 5e-3)  & 0.1141($\pm$ 7e-3)  & 0.3635($\pm$ 6e-3)  \\
Boyl              & 82.31 ($\pm$ 0.43)                         & 12.39 ($\pm$ 0.37) & 16.46 ($\pm$ 0.34) & 16.01($\pm$ 0.41)                         & 0.0567 ($\pm$ 7e-4) & 0.2345($\pm$ 6e-3)  & 0.2249($\pm$ 7e-3)  & 0.1201($\pm$ 6e-3)  & 0.3647($\pm$ 6e-3)  \\
SimCLR + CCL      & 82.37 ($\pm$ 0.39)                         & 11.88 ($\pm$ 0.36) & 15.90 ($\pm$ 0.37)  &14.89($\pm$ 0.35) & 0.0536 ($\pm$ 6e-4) & 0.2264($\pm$ 7e-3)  & 0.2187($\pm$ 6e-3)  & 0.1101($\pm$ 7e-3)  & 0.2893($\pm$ 7e-3)  \\ \hline
SoFCLR            & 82.36 ($\pm$ 0.31)                         & \textbf{9.71} ($\pm$ 0.27)  & \textbf{14.61} ($\pm$ 0.30) & \textbf{12.90}($\pm$ 0.31)                          & 0.0545 ($\pm$ 5e-4) & \textbf{0.2165}($\pm$ 5e-3)  & \textbf{0.2090}($\pm$ 3e-3)  & \textbf{0.1042}($\pm$ 6e-3)  & \textbf{0.1869}($\pm$ 5e-3) \\
        \bottomrule
    \end{tabular}}

    \centering
     \caption{Results on CelebA: accuracy of predicting Bags Under Eyes and fairness metrics for two sensitive attributes, Male and Young.}
    \label{tab:celeba_ncsf_bags_under_eyes}
    \resizebox{0.9\textwidth}{!}{
    \begin{tabular}{c|c|c|c|c|c|c|c|c|c} \toprule
   (Bags Under Eyes, Male) & Acc                 & $\Delta$ OD         & $\Delta$ EO         & $\Delta$DP          & IntraAUC            & InterAUC            & GAUC                & WD                  & KL                  \\ \toprule
CE                      & 81.49 ( $\pm$ 0.34) & 5.67 ( $\pm$ 0.32)  & 11.23 ( $\pm$ 0.33) & 7.09 ( $\pm$ 0.39)  & 0.0919 ($\pm$ 4e-3) & 0.2781 ($\pm$ 5e-3) & 0.2921 ($\pm$ 5e-3) & 0.1355 ($\pm$ 4e-3) & 0.6803 ($\pm$ 5e-3) \\
CE+EOD                  & 81.24 ( $\pm$ 0.33) & 6.71 ( $\pm$ 0.33)  & 11.12 ( $\pm$ 0.34) & 6.75 ( $\pm$ 0.36)  & 0.0848 ($\pm$ 5e-3) & 0.3045 ($\pm$ 6e-3) & 0.2933 ($\pm$ 6e-3) & 0.1347 ($\pm$ 6e-3) & 0.6767 ($\pm$ 6e-3) \\
CE + DPR                & 81.27 ( $\pm$ 0.41) & 6.39 ( $\pm$ 0.37)  & 10.06 ( $\pm$ 0.37) & 6.63 ($\pm$ 0.39)   & 0.1034 ($\pm$ 6e-3) & 0.2902 ($\pm$ 5e-3) & 0.2968 ($\pm$ 6e-3) & 0.1414 ($\pm$ 6e-3) & 0.6951 ($\pm$ 5e-3) \\
CE + EQL                & 81.36 ( $\pm$ 0.38) & 6.57 ( $\pm$ 0.35)  & 10.20 ( $\pm$ 0.40) & 6.95 ($\pm$ 0.43)   & 0.0937 ($\pm$ 3e-3) & 0.309 ($\pm$ 8e-3)  & 0.3024 ($\pm$ 6e-3) & 0.1443 ($\pm$ 5e-3) & 0.7278 ($\pm$ 7e-3) \\
ML-AFL                  & 81.74 ( $\pm$ 0.40) & 6.31 ( $\pm$ 0.33)  & 12.10 ( $\pm$ 0.39)  & 7.23 ( $\pm$ 0.33)  & 0.0963 ($\pm$ 5e-3) & 0.3231 ($\pm$ 6e-3) & 0.2392 ($\pm$ 5e-3) & 0.1449 ($\pm$ 5e-3) & 0.6883 ($\pm$ 6e-3) \\
Max-Ent                 & 81.36 ( $\pm$ 0.35) & 6.19 ( $\pm$ 0.34)  & 11.29 ($\pm$ 0.38)  & 6.83 ( $\pm$ 0.37)  & 0.0883 ($\pm$ 4e-3) & 0.2923 ($\pm$ 5e-3) & 0.2439 ($\pm$ 7e-3) & 0.1529 ($\pm$ 4e-3) & 0.6723 ($\pm$ 5e-3) \\ \hline
SimCLR                  & 82.14 ( $\pm$ 0.43) & 10.09 ( $\pm$ 0.31) & 16.10 ($\pm$ 0.41)  & 10.25 ( $\pm$ 0.39) & 0.0905 ($\pm$ 4e-3) & 0.3257 ($\pm$ 6e-3) & 0.3271 ($\pm$ 5e-3) & 0.1568 ($\pm$ 5e-3) & 0.8885 ($\pm$ 5e-3) \\
SogCLR                  & 81.63 ( $\pm$ 0.35) & 8.78 ( $\pm$ 0.30)  & 14.15 ( $\pm$ 0.33) & 9.83 ( $\pm$ 0.36)  & 0.0881 ($\pm$ 5e-3) & 0.3241 ($\pm$ 5e-3) & 0.3259 ($\pm$ 6e-3) & 0.1499 ($\pm$ 4e-3) & 0.8652 ($\pm$ 6e-3) \\
Boyl                    & 81.73 ( $\pm$ 0.38) & 8.63 ( $\pm$ 0.33)  & 13.23 ( $\pm$ 0.36) & 9.64 ( $\pm$ 0.37)  & 0.0923 ($\pm$ 6e-3) & 0.3345 ($\pm$ 5e-3) & 0.3149 ($\pm$ 6e-3) & 0.1423 ($\pm$ 6e-3) & 0.8211 ($\pm$ 6e-3) \\
SimCLR + CCL            & 81.58 ( $\pm$ 0.37) & 7.67 ( $\pm$ 0.35)  & 11.89 ( $\pm$ 0.37) & 8.92 ( $\pm$ 0.35)  & 0.0911 ($\pm$ 5e-3) & 0.2911 ($\pm$ 5e-3) & 0.3041 ($\pm$ 5e-3) & 0.1213 ($\pm$ 5e-3) & 0.7611 ($\pm$ 6e-3) \\ \hline
SoFCLR                  & 81.43 ( $\pm$ 0.29) & \textbf{5.19} ( $\pm$ 0.27)  & \textbf{7.47} ( $\pm$ 0.31)  & \textbf{6.53} ( $\pm$ 0.30)  & \textbf{0.0902} ($\pm$ 3e-3) & \textbf{0.2348} ($\pm$ 4e-3) & \textbf{0.2583} ($\pm$ 4e-3) & \textbf{0.0838} ($\pm$ 3e-3) & \textbf{0.4794} ($\pm$ 4e-3) \\
        \midrule
(Bags Under Eyes, Young) & Acc                                        & $\Delta$ ED        & $\Delta$ EO        & $\Delta$ DP        & IntraAUC            & InterAUC            & GAUC                & WD                  & KL                  \\ \hline
CE                       & 83.13 ($\pm$ 0.41)                         & 12.51 ($\pm$ 0.39) & 18.12 ($\pm$ 0.37) & 14.45 ($\pm$ 0.42) & 0.0376 ($\pm$ 8e-4) & 0.1861 ($\pm$ 6e-3) & 0.1842 ($\pm$ 8e-3) & 0.1195 ($\pm$ 5e-3) & 0.2391 ($\pm$ 6e-3) \\
CE + EOD                 & 83.03 ($\pm$ 0.38)                         & 10.96 ($\pm$ 0.37) & 14.84 ($\pm$ 0.38) & 14.23 ($\pm$ 0.39) & \textbf{0.0372} ($\pm$ 9e-4) & 0.1682 ($\pm$ 8e-3) & 0.1767($\pm$ 9e-3)  & 0.1131 ($\pm$ 4e-3) & 0.2177 ($\pm$ 7e-3) \\
CE + DPR                 & 82.67 ($\pm$ 0.37)                         & 8.32 ($\pm$ 0.43)  & 11.05 ($\pm$ 0.40) & 11.33 ($\pm$ 0.41) & 0.0413 ($\pm$ 8e-4) & 0.1622 ($\pm$ 7e-3) & 0.1745($\pm$ 9e-3)  & 0.1043 ($\pm$ 5e-3) & 0.2103 ($\pm$ 6e-3) \\
CE + EQL                 & 82.57 ($\pm$ 0.40)                         & 9.02 ($\pm$ 0.37)  & 11.68 ($\pm$ 0.37) & 12.33 ($\pm$ 0.38) & 0.0434 ($\pm$ 7e-4) & 0.1579 ($\pm$ 5e-3) & 0.1704($\pm$ 7e-3)  & 0.1037 ($\pm$ 5e-3) & 0.1991 ($\pm$ 6e-3) \\
ML-AFL                   & 81.91 ($\pm$ 0.38)                         & 8.08 ($\pm$ 0.41)  & 12.22 ($\pm$ 0.51) & 8.92 ($\pm$ 0.43)  & 0.0427 ($\pm$1e-3)  & 0.1926 ($\pm$ 7e-3) & 0.1823($\pm$ 8e-3)  & 0.0963 ($\pm$ 6e-3) & 0.2543 ($\pm$ 6e-3) \\
Max-Ent                  & 81.56 ($\pm$ 0.42)                         & 10.61 ($\pm$ 0.39) & 16.16 ($\pm$ 0.38) & 9.12 ($\pm$ 0.37)  & 0.0442 ($\pm$ 7e-4) & 0.1939 ($\pm$ 9e-3) & 0.1993($\pm$ 7e-3)  & 0.1993 ($\pm$ 4e-3) & 0.2737 ($\pm$ 7e-3) \\ \hline
SimCLR                   & 82.13 ($\pm$ 0.37)                         & 10.01 ($\pm$ 0.40) & 17.01 ($\pm$ 0.43) & 9.81 ($\pm$ 0.39)  & 0.0484 ($\pm$1e-3)  & 0.1901 ($\pm$ 8e-3) & 0.2011($\pm$ 6e-3)  & 0.1057 ($\pm$ 6e-3) & 0.2808 ($\pm$ 7e-3) \\
SogCLR                   & 81.63 ($\pm$ 0.38)                         & 8.13 ($\pm$ 0.38)  & 14.21 ($\pm$ 0.41) & 9.63 ($\pm$ 0.36)  & 0.0494 ($\pm$ 8e-4) & 0.1906 ($\pm$ 6e-3) & 0.1990($\pm$ 7e-3)  & 0.1029($\pm$ 4e-3)  & 0.2827 ($\pm$ 6e-3) \\
Boyl                     & 81.56 ($\pm$ 0.34) & 8.54($\pm$ 0.37)   & 15.32 ($\pm$ 0.42) & 9.71 ($\pm$ 0.37)  & 0.0467 ($\pm$ 7e-4) & 0.1924 ($\pm$ 6e-3) & 0.1948 ($\pm$ 6e-3) & 0.1037 ($\pm$ 5e-3) & 0.2748($\pm$ 7e-3)  \\
SimCLR + CCL             & 81.43 ($\pm$ 0.36) & 7.93 ($\pm$ 0.37)  & 13.91 ($\pm$ 0.39) & 9.13 ($\pm$ 0.34)  & 0.0443 ($\pm$ 7e-4) & 0.1877 ($\pm$ 7e-3) & 0.1849 ($\pm$ 8e-3) & 0.0837 ($\pm$ 5e-3) & 0.2564 ($\pm$ 6e-3) \\ \hline
SoFCLR                   & 81.62 ($\pm$ 0.33)                         & \textbf{6.91} ($\pm$ 0.34)  & \textbf{10.32} ($\pm$ 0.35) &  \textbf{7.89} ($\pm$ 0.31)  & 0.0377 ($\pm$ 5e-4) & \textbf{0.1729} ($\pm$ 6e-3) &  \textbf{0.1701} ($\pm$ 5e-3) &  \textbf{0.0565} ($\pm$ 3e-3) &  \textbf{0.1944} ($\pm$ 5e-3) \\
\bottomrule
    \end{tabular}}
\end{table*}

\begin{table}[h!]
    \centering
    \caption{UTKFace training data statistics for two different tasks with ethnicity and age as the sensitive attribute, respectively.}
    \resizebox{0.55\linewidth}{!}{
    \begin{tabular}{c|c|c||c|c|c} \toprule
     \textbf{ethnicity}   & Caucasian &  Others & \textbf{age} & $<=35$ & $> 35$\\ \toprule
       Female  & 1000  & 4000 & Female  & 3250 & 1750 \\
         Male      & 4000  & 1000 & Male &1750 &3250 \\ \bottomrule
    \end{tabular}}
    \label{tab:enthnicity_gender}
\end{table}

The convergence curves of our algorithm on UTKface data are shown in Figure~\ref{fig:conv}. We also plot the prediction score distributions for positive and negative class on UTKFace in Figure~\ref{fig:dis}. The results on other tasks of CelebA data are shown in Table~\ref{tab:celeba_ncsf_big_nose} and \ref{tab:celeba_ncsf_bags_under_eyes}.  

\begin{figure*}
\centering

           \includegraphics[width=0.32\linewidth]{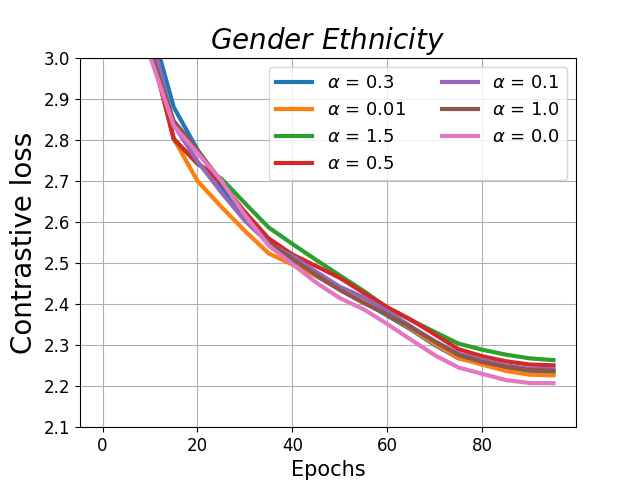}
        \includegraphics[width=0.32\linewidth]{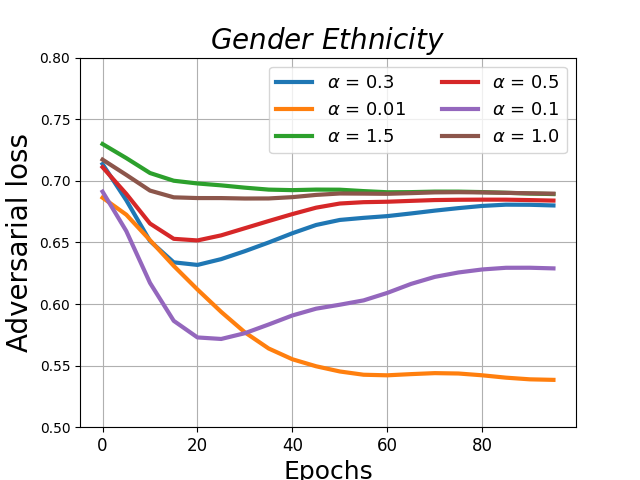}
           \includegraphics[width=0.32\linewidth]{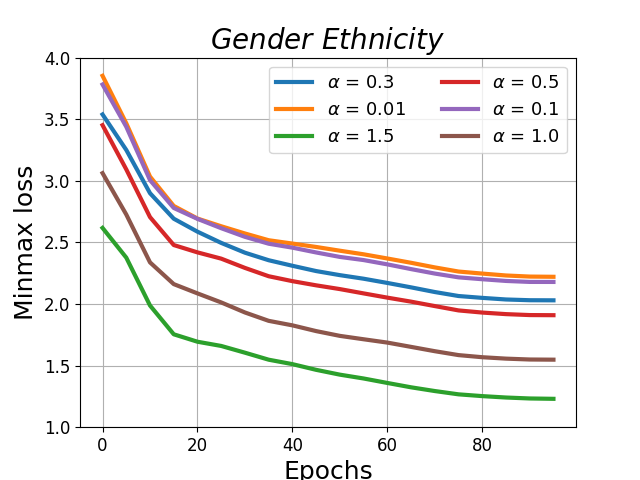} 
             \includegraphics[width=0.32\linewidth]{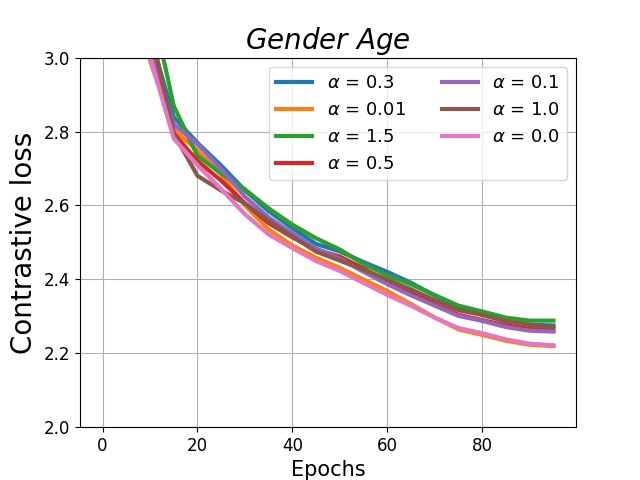}
        \includegraphics[width=0.32\linewidth]{NCSF_figures/max_gender_ib_3.png}
           \includegraphics[width=0.32\linewidth]{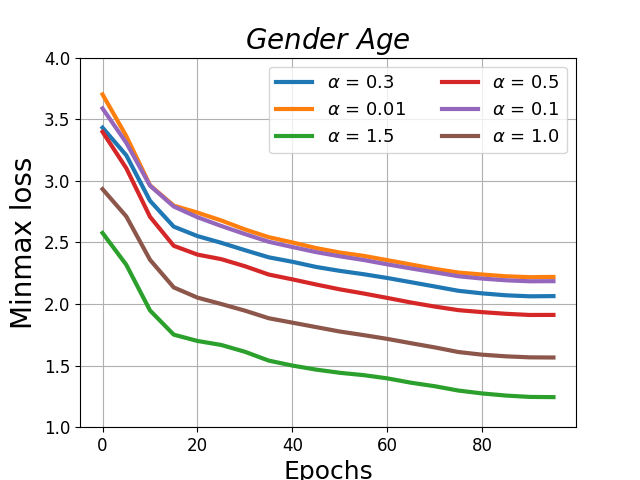}

      \caption{The convergence curves of different objective components optimized by SoFCLR with varying $\alpha$ values on the UTKFace dataset, using gender as the target label and different sensitive attributes, are shown in the figure.} \label{fig:conv}
\end{figure*}

\begin{figure*}
\centering


             \includegraphics[width=0.32\linewidth]{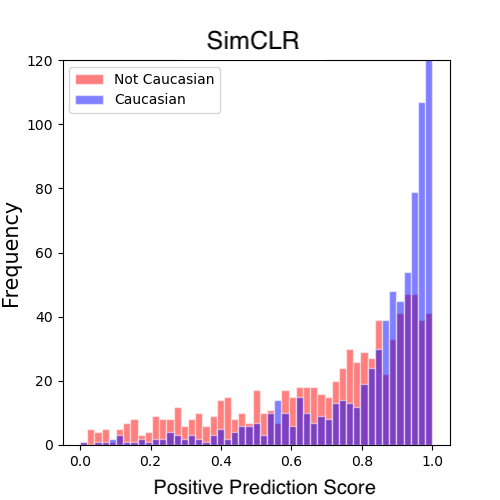}
  \includegraphics[width=0.32\linewidth]{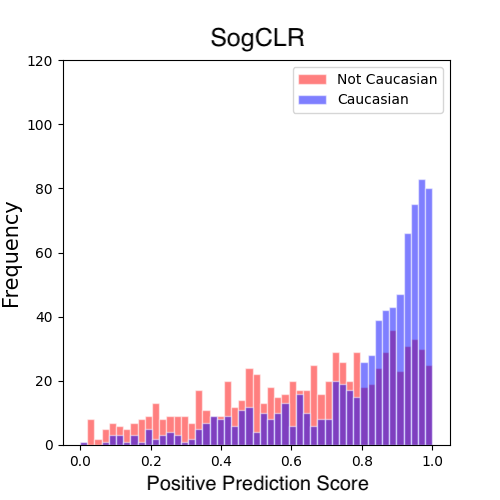}
         \includegraphics[width=0.32\linewidth]{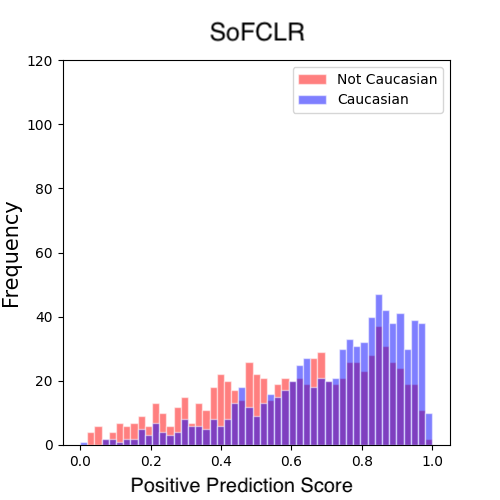}
           \\
        \includegraphics[width=0.32\linewidth]{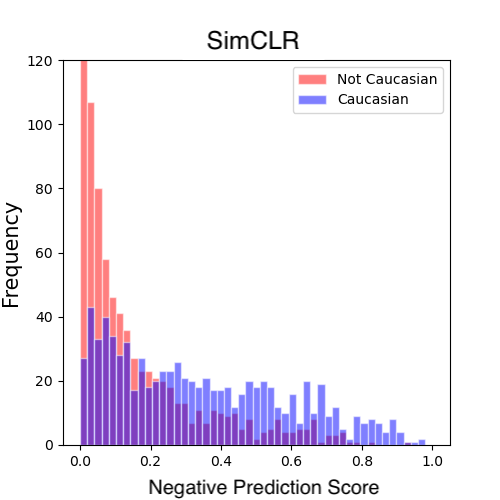}
        \includegraphics[width=0.32\linewidth]{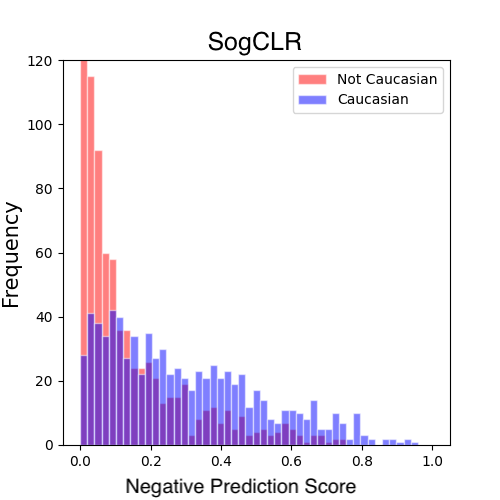}
          \includegraphics[width=0.32\linewidth]{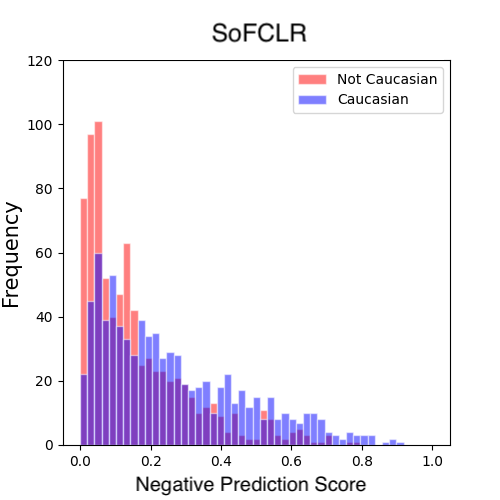}         
      \caption{Prediction score distributions for positive and negative class on UTKFace with gender as the target and ethnicity as the sensitive attribute. }\label{fig:dis}
\end{figure*}

\subsection{Ablation Studies on  Multi-valued Sensitive Attribute}

While adding experiments for multi-valued sensitive attribute will add value to this paper, this should not dim the major contribution of this paper. The analysis to multi-valued sensitive attribute is a straightforward extention. Here are some details. Indeed, the optimization algorithms presented in Section 5 is generic enough to cover the multi-valued sensitive attribute as long as the adversarial loss is replaced by the following one. Below, we present the analysis of fairness and also present some empirical results. For multi-valued sensitive attribute with $K$ possible values, we let $D(E(x))\in\mathbb R^K$ denote the predicted probabilities for $K$ possible values of the sensitive attribute. Denote by $D_k(E(x))$ as the $k$-th element of $D(E(x))$ for the $k$-th value of the sensitive attribute. Then $\sum_{k=1}^K D_k(E(x))=1$. We define the minimax problem as: \begin{align*} \min_{E}\max_{D}\mathbb E_{x,a}& \left[\sum_{k=1}^K\delta(a, k)\log D_k(E(x)) \right] \end{align*} Let us first fix $E$ and optimize $D$. The objective is equivalent to \begin{align*} \mathbb E_{x}\sum_{k=1}^Kp(a=k|E(x))\log D_k(E(x)) \end{align*} By maximizing $D(E(x))$, we have $D_k(E(x)) = p(a=k|E(x))$. Then we have the following objective for $E$: \begin{align*} \mathbb E_{x,a}& \left[\sum_{k=1}^K\delta(a, k)\log p(a=k|E) \right] \\ &= \mathbb E_{x, a} \left[\sum_{k=1}^K\delta(a, k)\log \frac{p(E|a=k)p(a=k)}{p(E)} \right]\\ & = \mathbb E_{x, a}\left[\sum_{k=1}^K\delta(a, k)\log p(a=k)\right] + \mathbb E_{x, a}\left[\sum_{k=1}^K\delta(a, k)\log \frac{p(E|a=k)}{p(E)}\right]\\ & = C + \mathbb E_{x, a}\left[\log \frac{p(E|a)}{p(E)} \right] = C + \mathbb E_{a}\mathbb E_{p(E|a)}\left[\log \frac{p(E|a)}{p(E)} \right] = C +\mathbb E_{a}[\text{KL}(p(E|a), p(E))] \end{align*} where $C$ is independent of $E$. Hence by minimizing over $E$ we have the optimal $E_*$ satisfying $p(E_*|a) = p(E_*)$. As a result, $D^*k(E(x)) = p(a=k|E_(x)) = p(a=k)$.

We have an experiment on UTKface dataset. We consider the sensitive attribute of age. Departing from the conventional binary division based on age 35, we stratify age into four distinct groups, delineated at ages 20, 35, and 60. Other settings are the same as in the paper. We compare SoFCLR with the baseline of SogCLR.
\begin{table}[h!]
\centering
\caption{Experimental results on UTKFace predicting binary task label Gender and fairness metrics for multi-valued Age attribute.}
\label{tab:multi-attr}
\resizebox{\columnwidth}{!}{
\begin{tabular}{c|ccccccccc}
\toprule
                       & Acc   & $\Delta ED$ & $\Delta EO$ & $\Delta DP$ & IntraAUC & InterAUC & GAUC   & WD     & KL     \\ 
                \midrule
SogCLR & 87.52 & 19.01       & 10.34       & 11.22       & 0.0910   & 0.0402   & 0.0515 & 0.1195 & 0.5832 \\
SoFCLR & 87.84 & \textbf{16.97}       & \textbf{9.84}        & \textbf{10.01}       & \textbf{0.0880}   & \textbf{0.0399}   & \textbf{0.0473} & \textbf{0.1503} & \textbf{0.5173} \\
\toprule
\end{tabular}
}
\end{table}

The results are reported in Table~\ref{tab:multi-attr}. The first three and the last two fairness metrics are computed in a similar way. We take $\Delta ED$ as an example. Given four groups, g0, g1, g2, g3. We compare pariwise fairness metrics and average over all pairs of values of the sensitive attribute, i.e, $\Delta ED= (\Delta ED(g0,g1) + \Delta ED(g1,g2)+ \Delta ED(g2,g3))/3$. For AUC fairness metrics, we convert it into four one-vs-all binary tasks and compuate averaged fairness metrics, {g0 vs not g0} + {g1 vs not g1}, {g2 vs not g2}, {g3 vs not g3}).








\clearpage

\subsection{Compare to VAE-based method.}
We have compared with one VAE based method from Louizos et al., 2015~\cite{DBLP:journals/corr/LouizosSLWZ15}. It is worth to mentioning that, no image-based VAE code was released by Louizos et al., 2015. For a fair comparison, we adopt a ResNet18-based Encoder-Decoder framework with an unsupervised setups where $y$ is unavailable and partial sensitive attribute $a$ is available. For the samples sensitive attribute labels are unavailable, we choose to train the model using the original VAE-type loss, specifically without the MMD (Maximum Mean Discrepancy) fairness regularizer. We training the losses for 100 epochs on the UTKface dataset, followed by conducting linear evaluations as in the paper.
\begin{table}[htbp]
\caption{Experimental results on UTKFace of predicting Gender and fairness metric for Ethnicity sensitive attribute.}
\label{tab:vae-compare}
\resizebox{\columnwidth}{!}{
\begin{tabular}{c|ccccccccc}
\toprule
& Acc                   & $\Delta ED$          & $\Delta EO$ & $\Delta DP$ & IntraAUC & InterAUC & GAUC   & WD     & KL                    \\ \midrule
 Louizos et al., 2015 & 60.05       & 4.79        & 6.68     & 9.6      & 0.0237 & 0.0755 & 0.0182 & 0.003  & 0.0129 \\
 SoFCLR               & \textbf{84.42}       & 13.02       & 13.23    & 13.00    & 0.0084 & 0.1013 & 0.1029 & 0.1195 & 0.1237 \\ \toprule
\end{tabular}
}

\end{table}

The results are reported on Table \ref{tab:vae-compare}.
We can see that the accuracy of the VAE-based method is much worse than our method by 24\%. On the other hand, its fairness metrics are better. This is expected due to the tradeoff between accuracy and fairness.

\end{document}